\documentclass[10pt,journal]{IEEEtran}
\IEEEoverridecommandlockouts
\usepackage{booktabs}       
\usepackage{pbox,adjustbox}
\usepackage{multirow}
\usepackage{array}

\usepackage{amsmath}
\usepackage{amsthm}
\usepackage{amssymb}
\usepackage{amsfonts}
\usepackage{mathrsfs}
\usepackage[ruled,vlined]{algorithm2e}
\usepackage{adjustbox}
\usepackage{graphicx}
\usepackage{subcaption}
\usepackage{cite}
\usepackage{bm,epsfig,amsthm,url}
\usepackage{indentfirst}
\usepackage{epstopdf}
\usepackage{color}
\definecolor{color1}{RGB}{128, 0, 0}
\definecolor{color2}{RGB}{0, 0, 128}
\definecolor{color3}{RGB}{0, 128, 0}

\usepackage[colorlinks, linkcolor=color1, anchorcolor=blue, citecolor=color1, urlcolor = color2]{hyperref}

\usepackage{verbatim} 

\newtheorem{remark}{Remark}

\newcommand{\equa}[1]{
\begin{equation}
\begin{aligned}
	#1
\end{aligned}	
\end{equation}
}

\newcommand{\snorm}[1]{
\left\|
	#1
\right\|^2
}

\newcommand{\bracket}[1]{
\left(
	#1
\right)
}
\newcommand{\sbracket}[1]{
\left[
	#1
\right]
}

\newtheorem{theorem}{Theorem}
\newtheorem{lemma}{Lemma}
\newtheorem{assump}{Assumption}
\newtheorem{corollary}{Corollary}
\newtheorem{fact}{Fact}

\begin{document}
\title{Communication-Efficient Stochastic Zeroth-Order Optimization for Federated Learning}
\author{Wenzhi Fang, Ziyi Yu, Yuning Jiang,~\textit{Member, IEEE}, Yuanming Shi,~\textit{Senior Member, IEEE}, \\ Colin N. Jones,~\textit{Senior Member, IEEE}, and Yong Zhou,~\textit{Senior Member, IEEE}
\thanks{Manuscript received Jan. 23, 2022; revised Jun. 3, 2022 and Sept. 22, 2022; accepted Oct. 06, 2022. The associate editor coordinating the review of this manuscript and approving it for publication was Dr. Mingyi Hong. The work of Yong Zhou was supported by the National Natural Science Foundation of China (NSFC) under Grants U20A20159 and 62001294. The work of Yuning Jiang was supported by the Swiss National Science Foundation under the RISK project (Risk Aware Data-Driven Demand Response),
	grant number 200021175627. (Corresponding author: Yong Zhou.)
}
\thanks{
	Wenzhi Fang, Ziyi Yu, Yuanming Shi, and Yong Zhou are with the School of Information Science and Technology, ShanghaiTech University, Shanghai 201210, China (e-mail: \{fangwzh1, yuzy, shiym, zhouyong\}@shanghaitech.edu.cn).
}
\thanks{
	Yuning Jiang and Colin N. Jones are with Automatic Control Laboratory, EPFL, Switzerland (e-mail: \{yuning.jiang, colin.jones\}@epfl.ch).
}
}

\maketitle

\begin{abstract}
Federated learning (FL), as an emerging edge artificial intelligence paradigm, enables many edge devices to collaboratively train a global model without sharing their private data. 
To enhance the training efficiency of FL, various algorithms have been proposed, ranging from first-order to second-order methods.
However, these algorithms cannot be applied in scenarios where the gradient information is not available, e.g., federated black-box attack and federated hyperparameter tuning. To address this issue, in this paper we propose a derivative-free federated zeroth-order optimization (FedZO) algorithm featured by performing multiple local updates based on stochastic gradient estimators in each communication round and enabling partial device participation. 
Under non-convex settings, we derive the convergence performance of the FedZO algorithm on non-independent and identically distributed data and characterize the impact of the numbers of local iterates and participating edge devices on the convergence.
To enable communication-efficient FedZO over wireless networks, we further propose an over-the-air computation (AirComp) assisted FedZO algorithm.
With an appropriate transceiver design, we show that the convergence of AirComp-assisted FedZO can still be preserved under certain signal-to-noise ratio conditions.
Simulation results demonstrate the effectiveness of the FedZO algorithm and validate the theoretical observations. 
	
\end{abstract}



\section{Introduction}

With the rapid advancement of the Internet of Things (IoT), a massive amount of data is generated and collected by various edge devices (e.g., sensors, smart phones). Because of the limited radio spectrum resource and increasing privacy concerns, gathering geographically distributed data from a large number of edge devices into a cloud server to enable cloud artificial intelligence (AI) may not be practical. To this end, edge AI has recently been envisioned as a promising AI paradigm \cite{yuanming_jsac}. 
Unlike cloud AI that relies on a cloud server to conduct centralized training, edge AI exploits the computing power of multiple edge devices to perform model training with their own local data in a distributed manner. 
Federated learning (FL) \cite{mcmahan17a}, as a representative edge AI framework, enables multiple edge devices to collaboratively train a shared model without exchanging their local data, which effectively alleviates the communication burden and privacy concerns.  Nowadays, FL has found application in various fields, including autonomous driving \cite{shi2021mobile}, recommendation systems \cite{recommend}, healthcare informatics \cite{Fedyang_network}, etc.

As a result of the popularity of FL, the federated optimization problem for model training has attracted a growing body of attention from both academia and industry in recent years. Various algorithms have been proposed to attain a fast convergence rate and reduce the  communication load, including both first- (e.g., FedAvg \cite{mcmahan17a}, FedPD \cite{fedpd}, FedNova \cite{jianyu}) and second-order algorithms (e.g., FedDANE \cite{fedDANE}). 
Most existing algorithms rely on gradient and/or Hessian information to solve the federated optimization problem. 
However, such information cannot be obtained in scenarios where the analytic expressions of the loss functions are unavailable, such as federated hyperparameter tuning \cite{bayesian} or distributed black-box attack of deep neural networks (DNN) \cite{yi2021zeroth}. 
In other words, existing algorithms cannot tackle federated optimization problems when gradient information is not available. 
This motivates us to develop a communication-efficient federated zeroth-order optimization algorithm that does not require gradient or Hessian information.



Parallel with the research on algorithm design for FL, the implementation of FL over wireless networks is also an emerging research topic. Random channel fading and receiver noise raise unique challenges for the training of FL over wireless networks. 
Guaranteeing the learning performance with limited radio resource is a challenging task, which requires the joint design of the learning algorithm and communication strategy. Along this line of research,  
the authors in
 \cite{guanding_schedule} studied the joint resource allocation and edge device selection to enhance learning performance.
Both studies adopted the orthogonal multiple access (OMA) scheme, where the number of edge devices that can participate in each communication round is restricted by the number of available time/frequency resource blocks. 
The limited radio resource turns out to be the main performance bottleneck of wireless FL. 
Fortunately, over-the-air computation (AirComp), as a non-orthogonal multiple access scheme, allows concurrent transmissions over the same radio channel to enable low-latency and spectrum-efficient wireless data aggregation \cite{zhibin_air,kaibin_WC,uav_min}, thereby mitigating the communication bottleneck \cite{me_t}. 
Motivated by this observation, various AirComp-assisted FL algorithms were proposed in \cite{yangkai_air,amiri2020machine,Broadband} to achieve fast model aggregation, wherein all of them adopted first-order optimization algorithm.
There still lacks a thorough investigation on AirComp-assisted FL with zeroth-order optimization.

\subsection{Main Contributions}

In this paper, we consider a federated optimization problem, where the gradient of the loss function is not available. We propose a derivative-free federated zeroth-order optimization algorithm, named FedZO, whose key features are performing multiple local updates based on a stochastic gradient estimator in each communication round and enabling partial device participation. 
We establish a convergence guarantee for the proposed FedZO algorithm and study its implementation over wireless networks with the assistance of AirComp, which is challenging for the following reasons. 
First, although executing multiple local iterates in each communication round reduces the communication overhead, it also increases the discrepancies among local models due to the data heterogeneity and may even lead to algorithmic divergence. To reduce the communication overhead while preserving the convergence, the relationship between the convergence behavior and the number of local iterates needs to be characterized.
Second, the stochastic gradient estimator adopted in FedZO is not an unbiased estimate of the actual gradient, as demonstrated in\cite{liu2018zeroth, GaoJZ18}.
These unique features together with multiple local iterates and partial device participation per communication round make the existing convergence analysis framework for FedAvg not applicable to the proposed FedZO algorithm. 
Third, to characterize the convergence of the AirComp-assisted FedZO algorithm, the impact of the random channel fading and receiver noise in global model aggregation needs to be further taken into account. This not only complicates the convergence analysis but also poses a new problem for the communication strategy design. 
In this paper, we develop a unified convergence analysis framework to address the aforementioned challenges. 
The main contributions of this paper are summarized as follows:

\begin{itemize}
\item We develop the derivative-free FedZO algorithm, which inherits the framework of the FedAvg algorithm but only queries the values of the objective function, to handle federated optimization problems without using gradient or Hessian information. To cater for the FL system with a large number of edge devices and to reduce the communication overhead, the proposed FedZO algorithm enables partial device participation and performs multiple local iterates in each communication round.  
\item We establish a convergence guarantee for the proposed FedZO algorithm under non-convex settings and on non-independent and identically distributed (non-i.i.d.) data, and then derive the maximum number of local iterates required for preserving convergence.
We demonstrate that the proposed FedZO algorithm can attain linear speedup in the number of local iterates and the number of participating edge devices. 
\item We study the implementation of the FedZO algorithm over wireless networks with the assistance of AirComp for the aggregation of local model updates in the uplink. With an appropriate transceiver design that can mitigate the impact of the fading and noise perturbation, we study the convergence behavior of the AirComp-assisted FedZO algorithm and characterize the impact of the signal-to-noise ratio (SNR) on the convergence performance.
\end{itemize}
We conduct extensive simulations to evaluate the performance of the proposed FedZO and AirComp-assisted FedZO algorithms.
Simulation results show that the proposed FedZO algorithm is convergent under various parameter settings and outperforms existing distributed zeroth-order methods.
Moreover, simulations illustrate that the performance of the proposed FedZO algorithm is comparable to that of the FedAvg algorithm, which indicates that our proposed algorithm can serve as a satisfactory alternative for FedAvg when first-order information is not available. 
Results also confirm that, with an appropriate SNR setting,  the AirComp-assisted FedZO algorithm preserves convergence.

\subsection{Related Works}

The study of FL started from the seminal work \cite{mcmahan17a}, where the authors proposed a communication-efficient federated optimization algorithm known as FedAvg. Subsequently, various articles established convergence guarantees for the FedAvg algorithm \cite{yu2019parallel,LiHYWZ20,wang2021field}.
Following the FedAvg algorithm, many other first-order methods have been proposed, e.g., FedPD \cite{fedpd}, FedNova \cite{jianyu}, FedProx \cite{fed_prox}, SCAFFOLD \cite{scaffold}, and FedSplit \cite{fedsplit}.
To further reduce the communication overhead, several second-order optimization algorithms were proposed, such as FedDANE \cite{fedDANE} and GIANT \cite{giant}. Although the aforementioned first- and second-order algorithms have broad applications, there are still many FL tasks where the gradient and Hessian information are unavailable and thus require zeroth-order optimization.

Recently, several works \cite{yi2021zeroth,zone_hong,tang2020distributed,DSZO,Anit_frank_wolfe,wenjie,Anit_random} focused on studying distributed zeroth-order optimization.
 Specifically, the authors in \cite{zone_hong} developed a so-called ZONE-S algorithm based on the primal-dual technique. 
 In \cite{tang2020distributed}, the authors employed the gradient tracking technique to develop a fast distributed zeroth-order algorithm. 
However, ZONE-S requires $\mathcal{O}\bracket{T}$ ($T$ denotes the number of total iterations) sampling complexity per iteration while the algorithm proposed in \cite{tang2020distributed} considered the deterministic setting. 
More recently, the authors in \cite{yi2021zeroth} proposed an algorithm with $\mathcal{O}\bracket{1}$ sampling complexity per iteration, which attains linear speedup in the number of edge devices. 
\cite{DSZO} proposed a decentralized zeroth-order algorithm that allows multiple local updates. However, the theoretic analysis in \cite{DSZO} focused on the strongly convex scenario and relied on the Lipschitzness of local functions, which is relatively restrictive  \cite{tighter_theory}. 
The authors in \cite{Anit_frank_wolfe} proposed and analyzed a distributed zeroth-order Frank-Wolfe algorithm for constrained optimization. Based on single-point and Kiefer-Wolfowitz type gradient estimators, the authors in \cite{wenjie,Anit_random} proposed two distributed zeroth-order algorithms over time-varying graphs. 
	It is worth noting that the aforementioned works mainly consider the peer-to-peer architecture while the studies on the central-server-based architecture are very limited.
Moreover, most of the existing distributed zeroth-order algorithms focused on full device participation, which may not be practical for FL systems with limited radio resources and a large number of edge devices \cite{LiHYWZ20}. 

AirComp has recently been adopted to support the implementation of FL over wireless networks \cite{yangkai_air,zhangyingjun_FL,xiaowen_jsac,zhibin_fl,yonina_fl,Accelerated_cobin}, where channel fading and receiver noise inevitably distort the model aggregation, and in turn introduce a detrimental impact on the learning performance \cite{yangkai_air}.
The convergence behavior of various FL algorithms, e.g., vanilla gradient method \cite{zhangyingjun_FL} and stochastic gradient method \cite{xiaowen_jsac}, showed that the channel fading and noise perturbation typically introduce a non-diminishing optimality gap, which can be mitigated by transmit power control \cite{xiaowen_jsac}, beamforming design \cite{zhangyingjun_FL}, and device scheduling \cite{zhibin_fl}.
In \cite{yonina_fl}, the authors proposed a  joint learning and transmission scheme to ensure global convergence for strongly convex problems. By utilizing the communication strategy in \cite{yonina_fl}, the authors in \cite{Accelerated_cobin} developed an AirComp-assisted accelerated gradient descent algorithm.
Despite the above progress, the existing works focused on the first-order method, while there is no relevant literature studying the AirComp-assisted zeroth-order optimization algorithm.


\subsection{Organization}
The remainder of this paper is organized as follows.
We present the problem formulation and propose a federated zeroth-order optimization algorithm in Section \ref{fedrated_zero}. 
Section \ref{theory} provides the convergence analysis.
Section \ref{aircomp} studies the implementation of the proposed FedZO algorithm over wireless networks using AirComp. 
The simulation results are provided in Section \ref{simulation}. 
Finally, we conclude this paper in Section \ref{conclusion}.

\noindent \textbf{Notation:}	
We denote the $\ell_2$ norm of vectors by $\|\cdot\|$. $[T]$ denotes the set $\{1,2,\ldots,T-1\}$. $\mathbb{S}^{d} = \{ \bm v \in \mathbb{R}^d \mid \|\bm v\| =1 \}$ denotes a $d$-dimensional unit sphere. $\mathbb{B}^{d} = \{ \bm v \in \mathbb{R}^d \mid \|\bm v\| \leq1 \}$ denotes a $d$-dimensional unit ball.  We denote uniform distributions over $\mathbb{S}^{d}$ and $\mathbb{B}^d$ by $\mathcal{U}(\mathbb{S}^{d})$ and $\mathcal{U}(\mathbb{B}^d)$, respectively. We denote $\sim$ as the uniform sampling.
For a function $F$, $\nabla F$ and $\widetilde{\nabla}F$ denote the gradient and gradient estimator, respectively.

\section{Federated Zeroth-Order Optimization} \label{fedrated_zero}

In this section, we first introduce the federated optimization problem and then propose a federated zeroth-order optimization algorithm.

\subsection{Problem Formulation}


Consider an FL task over a network consisting of a central server and $N$ edge devices indexed by $\{1,2,\ldots,N\}$. The goal of the central server is to coordinate all edge devices to collaboratively solve the following federated optimization problem 
\equa{
	\label{global_FL}
	\min_{\bm x \in \mathbb{R}^d}  f(\bm x)  \triangleq \frac{1}{N} \sum_{i=1}^{N} f_i\left(\bm x\right),
}
where $\bm x \in \mathbb{R}^d$ denotes the model parameter of dimension $d$, and $f_i(\bm x)$ and $f(\bm x)$ denote the local loss function of edge device $i$ and the global loss function at the central server evaluated at model parameter $\bm x$, respectively. We assume that each edge device with a local dataset is equally important for the global model \cite{wang2021field}. 
In \eqref{global_FL}, $f_i (\bm x)$ measures the expected risk over the local data distribution denoted as $\mathcal{D}_i$ at edge device $i$, given by
\[
 f_i\left(\bm x\right) \triangleq \mathbb{E}_{ \xi_i \sim \mathcal{D}_i}[F_i(\bm x, \xi_i)],
\]
where $F_i\left(\bm x, \xi_i\right)$ represents the loss with respect to $\xi_i$ evaluated at model parameter $\bm x$ and $\xi_i$ denotes a random variable uniformly distributed over $\mathcal{D}_i$. In particular, a realization of $\xi_i$ is a single data sample.
With sampling information of $\xi_i$ and $\bm x$, edge device $i$ can query the function value of $F_i$,
which serves as a stochastic approximation of the expected loss $f_i(\bm x)$. Note that the analytic expression and gradient information of $F_i$ are not available. 


\begin{remark}
The scenarios where the gradient information is not available arise in many practical applications \cite{liu2020primer}, including but not limited to federated black-box attacks of DNN \cite{yi2021zeroth} and federated hyperparameter tuning in model training \cite{bayesian}. 
To be specific, in federated black-box attacks, the gradient information cannot be acquired as the deep model is hidden. In the federated hyperparameter tuning task, there does not exist an analytic relationship between the training loss and the hyperparameters. 	
\end{remark}	

\subsection{Preliminaries on Stochastic Gradient Estimator}	
We adopt a mini-batch-type stochastic gradient estimator \cite{ZOOADMM}. Specifically, for function $F_i$,  the mini-batch-type stochastic gradient estimator is given by
	\begin{align}\label{stochastic_estimator_ori}
	&\widetilde{\nabla} F_i \left (\bm x,  \{\xi_{i,m}\}_{m=1}^{b_1}, \{\bm v_{i,n}\}_{n=1}^{b_2}, \mu \right)  \nonumber \\
	=&
	\frac{1}{b_1 b_2}\!\sum_{m=1}^{b_1}\! \sum_{n=1}^{b_2}\! \frac{d \bm v_{i,n}}{\mu} \!
	\Big(\!
	F_i (\bm x \!+\! \mu \bm v_{i,n},   \xi_{i,m} ) \!-\! F_i (\bm x,  \xi_{i,m})
	\! \Big),
	\end{align}
	where $\{\xi_{i,m}\}_{m=1}^{b_1}$ is a sequence of independent and identically distributed (i.i.d.) random variables with the same distribution as $\xi_i$, $\{\bm v_{i,n}\}_{n=1}^{b_2}$ is a sequence of i.i.d. random vectors with distribution $\mathcal{U}(\mathbb{S}^{d})$, and $\mu$ is a positive step size. 
It has been shown in \cite{GaoJZ18} that
	\begin{align}
	\mathbb{E}\!\left[\!\frac{d \bm v_{i,n}}{\mu} \!\Big(\!
	F_i (\bm x \!+\! \mu \bm v_{i,n},   \xi_{i,m} ) \!-\! F_i (\bm x,  \xi_{i,m})
	\! \Big)\!\right] \!=\!\nabla f_i^{\mu} (\bm x), \label{single_equal}
	\end{align} 
	where the expectation is taken over $\{\xi_{i,m}, \bm v_{i,n}\}$ and $f_i^{\mu}(\bm x) =  \mathbb{E}_{\bm u\sim\mathcal{U}(\mathbb{B}^d)}\left[
	f_i\left(\bm x + \mu \bm u \right)\right]$ is a locally averaged version of $f_i(\bm x)$. 
	Furthermore, we have 
		\begin{align}
		\mathbb{E} \left[\widetilde{\nabla} F_i \left(\bm x,  \{\xi_{i,m}\}_{m=1}^{b_1}, \{\bm v_{i,n}\}_{n=1}^{b_2}, \mu \right)
		\right] = \nabla f_i^{\mu} (\bm x). \label{mini_batch_equal}
		\end{align}
	where the expectation is taken over $\{\xi_{i,m}\}_{m=1}^{b_1}$ and $\{\bm v_{i,n}\}_{n=1}^{b_2}$.
As $\nabla f_i^{\mu} (\bm x)$ is a biased approximation of $\nabla f_i(\bm x)$ \cite{liu2020primer}, \eqref{stochastic_estimator_ori} is a biased estimate of the actual gradient.

	 The mini-batch-type stochastic gradient estimator enjoys a low variance than the two-point stochastic gradient estimator \cite{zone_hong, Anit_frank_wolfe, ZOOADMM}. 
	 Besides, the mini-batch sizes $b_1$ and $b_2$ are independent of $d$, and hence the computation complexity of \eqref{stochastic_estimator_ori} does not scale with the  dimension of variable $\bm x$.

\subsection{FedZO Algorithm}  

Inspired by the FedAvg algorithm, we develop a federated zeroth-order optimization algorithm summarized in Algorithm \ref{algo:fed_zero}.
The main idea of the FedZO algorithm is to get rid of the dependence on the gradient and reduce the frequency of model exchanges, which are achieved by employing a gradient estimator \eqref{stochastic_estimator_ori} and performing $H$ steps of stochastic zeroth-order updates per communication round, respectively.
The FedZO algorithm consists of the following four phases in each round.
\begin{itemize}  
\item 
\emph{Global Model Dissemination:} At the beginning of the $t$-th round, the central server uniformly samples $M$ edge devices to participate in the local training. The set of scheduled edge devices in round $t$ is denoted as $\mathcal{M}_t$. Then, the central server disseminates its current global model parameter $\bm x^{t}$ to the selected edge devices.
\item 
\emph{Local Model Update:} After receiving the model parameter $\bm x^t $ from the central server, each edge device $i\in \mathcal{M}_t$ initializes its local model $\bm x_i^{(t,0)}$ with the received global model from the central server, i.e., $\bm x_i^{(t,0)} = \bm x^{t}$, and then takes a total of $H$ iterates of stochastic zeroth-order updates. 
		In particular, at the $k$-th iteration of the $t$-th round, 
		edge device $i$ computes a stochastic gradient estimator according to \eqref{stochastic_estimator_ori}. 
For notational ease, we denote it as
\begin{align}\label{stochastic_estimator}
\bm e_i^{(t,k)} \!= \!\widetilde{\nabla} F_i \left (\bm x_i^{(t,k)}\!,\!  \{\xi_{i,m}^{(t,k)}\}_{m=1}^{b_1}\!,\! \{\bm v_{i,n}^{(t,k)}\}_{n=1}^{b_2}, \mu \right), 
\end{align}
where $\bm x_i^{(t,k)}$ represents the local model of edge device $i$ at the $k$-th iteration of the $t$-th round. 
Subsequently, the sampled edge devices update their local models by performing the following  stochastic zeroth-order update
\begin{align}\label{updata_equation}
\bm x_i^{(t,k\!+\!1)} \!=\! \bm x_i^{(t,k)}\! - \eta \bm e_i^{(t,k)},   k= 0,\! 1,\!\ldots \!, H\!-\!1,
\end{align}
where $\eta$ denotes the learning rate. 
After $H$ local iterates, edge device $i$ obtains an updated local model $\bm x_i^{(t,H)}$.
\item 
\emph{Local Model Uploading:} All edge devices in set $\mathcal{M}_t$ calculate the updates of their local models in this round, i.e., $\bm \Delta_i^{t} = \bm x_i^{(t,H)} - \bm x_i^{(t,0)}, ~ i \in \mathcal{M}_t$, and then upload these updates to the central server. 
\item 
\emph{Global Model Update:} After receiving local model updates from the sampled edge devices,  the central server aggregates these updates, i.e., $\bm \Delta^{t} = \frac{1}{M} \sum_{i \in \mathcal{M}_t} \bm \Delta_i^{t} $, and then updates the global model, i.e., $\bm x^{t+1} = \bm x^{t} + \bm \Delta^{t}$.
\end{itemize} 

\begin{algorithm}[t!]
    \DontPrintSemicolon
    \SetKwInput{Input}{Input}
    \SetAlgoLined
    \Input{Initial model $\bm x^{0}$, learning rate $\eta$, step size $\mu$, mini-batch sizes $b_1$, $b_2$,} number of participating edge devices $M$
     \For{$t \in \{0,1,\dots,T-1\}$ }{ 
     	Uniformly sample a subset $\mathcal{M}_t$ of $M$ edge devices\;    
     Disseminate global model $\bm x^{t}$ to all edge devices in set $\mathcal{M}_t$\;
      \For{\textbf{edge device} $i\in \mathcal{M}_t$  \textbf{in parallel}}{
        Initialize local model $\bm x_i^{(t,0)}=\bm x^{t}$\;
        \For {$k =0,\dots,H-1$}{
             Generate $\{\xi_{i,m}^{(t,k)}\}_{m=1}^{b_1} \sim \mathcal{D}_i$ independently \;
            Generate $\{\!\bm v_{i,n}^{(t,k)}\!\}_{n=1}^{b_2} \!\sim\! \mathcal{U}(\mathbb{S}^{d})$ independently \;
            \vspace{-4mm}
            Update $\bm x_i^{(t,k\!+\!1)} $ by \eqref{updata_equation}\;
        }
        Compute local model updates $\bm \Delta_i^{t} = \bm x_i^{(t,H)} - \bm x_i^{(t,0)}$\;
        Upload local model updates to central server\;
      }
      Aggregate local changes $\bm \Delta^{t} = \frac{1}{M} \sum_{i \in \mathcal{M}_t} \bm \Delta_i^{t} $\;
      Update global model $\bm x^{t+1} = \bm x^{t} + \bm \Delta^{t}$\;
     }
     \caption{FedZO Algorithm}
     \label{algo:fed_zero}
\end{algorithm}

Although the proposed FedZO algorithm adopts a similar framework as the FedAvg algorithm, the convergence analysis of the latter cannot be directly extended to that of the FedZO algorithm. The key factor hindering the extension is that the gradient estimator does not preserve specific properties of the stochastic gradient. For instance, the gradient estimator \eqref{stochastic_estimator} is not an unbiased estimate of the true gradient. 
Besides, the existing theoretical analysis framework for the distributed zeroth-order optimization method cannot be applied to the FedZO algorithm as existing zeroth-order algorithms \cite{zone_hong,tang2020distributed,yi2021zeroth} do not consider multiple steps of local model updates and partial device participation. 
A larger number of local iterates reduces the communication overhead,  but also increases the local model discrepancies and may even lead to algorithmic divergence. 
To preserve convergence for the developed FedZO algorithm, it is necessary to bound these discrepancies by appropriately choosing the number of local updates, i.e., $H$. 
In Section \ref{theory}, we will provide the convergence analysis for the FedZO algorithm.

\section{Convergence Analysis for FedZO}\label{theory}
In this section, we present the convergence analysis of the FedZO algorithm with full and partial device participation. 
To make our analysis applicable for more practical scenarios, we focus on the settings of non-convex loss functions and the non-i.i.d. data.
We make the following assumptions for the tractability of convergence analysis. 
\begin{assump}\label{bounded}
The global loss in \eqref{global_FL}, i.e., $f(\bm x)$, is lower bounded by $f_*$, i.e., $f(\bm x) \geq f_* >  -\infty$.
\end{assump}
\begin{assump}
\label{assump_smooth} 
$F_i \left(\bm x,  \xi_{i} \right)$, $f_i(\bm x)$, and $f(\bm x)$ are $L$-smooth. 
Mathematically, for any $\bm x \in \mathbb{R}^d$ and $\bm y\in \mathbb{R}^d$, we have
	\equa{
	\| & \nabla f_i(\bm y) - \nabla f_i(\bm x) \| \leq L \| \bm y - \bm x \|, ~ \forall i ,\\
	f\left(\bm y \right)  &\leq f\left(\bm x \right) + \left\langle \nabla f(\bm x), \bm y - \bm x \right\rangle  + \frac{L}{2} \snorm{\bm y - \bm x}. \nonumber 
	}
\end{assump}
\begin{assump}
\label{bound_variance}
The second-order moment of stochastic gradient $\nabla F_i\left(\bm x, \xi_{i} \right)$ satisfies $\mathbb{E}_{\xi_{i}} \|\nabla F_i\left(\bm x, \xi_{i} \right)\|^2 \leq c_g \|\nabla f_i(\bm x)\|^2 +\sigma_g^2, ~ \forall \bm x \in \mathbb{R}^d, ~ \forall i$, where $c_g \geq 1$.	
\end{assump}
\begin{assump}
\label{bound_heterogeneity}
The gradient dissimilarity between each local loss function and the global loss function is bounded as $\|\nabla f\left(\bm x \right) - \nabla f_i(\bm x)\|^2 \leq c_h\|\nabla f(\bm x)\|^2 + \sigma_h^2 , ~ \forall \bm x \in \mathbb{R}^d, ~ \forall i$, where $c_h$ is a positive constant.
\end{assump}
Assumptions  \ref{bounded}-\ref{bound_variance} are commonly used in stochastic optimization \cite{bottou2018optimization}. Assumption \ref{bound_heterogeneity}, also known as the bounded gradient dissimilarity assumption \cite{yi2021zeroth}, is adopted to characterize the non-i.i.d. extent of the local data distribution. Similar assumptions have also been made in the literature \cite{jianyu,yu2019parallel,LiHYWZ20,wang2021field,fed_prox,scaffold} for the convergence analysis under the non-i.i.d. setting. 
Note that these assumptions are only required for convergence analysis which are standard in zeroth-order optimization \cite{yi2021zeroth,GaoJZ18}.

In the following, we first present the convergence analysis for full device participation and then extend the analysis to partial device participation.

\subsection{Full Device Participation} 
We first characterize the convergence of the FedZO algorithm with full device participation in Theorem \ref{theorem_full}. 
We take the squared gradient $\|\nabla f(\bm x^t) \|^2$ to evaluate the suboptimality of the iterates. 
The speed of approaching a stationary point is an important metric to evaluate the algorithmic effectiveness for non-convex problems \cite{nesterov2017random}.

\begin{theorem} \label{theorem_full}
	Suppose Assumptions \ref{bounded}-\ref{bound_heterogeneity} hold and the learning rate satisfies
	\equa{\label{theorem_eta_full}
\eta \leq  \min \left \{\frac{N}{72\tilde{c}_g\tilde{c}_hL}, \frac{2}{NH^2L}, \frac{1}{3\sqrt{\tilde{c}_g}HL} \right\},
} 
the FedZO algorithm with full device participation satisfies 
	\begin{align}\label{equation_full}
		\min_{t\in [T]}  
		\mathbb{E} \left\|\nabla f\left(\bm x^{t}\right)\right\|^2 
		\leq & 
		4\frac{ f\left(\bm x^{0}\right)  - f_* }{ HT \eta} + \eta \frac{24 L}{N}   \tilde{\sigma}^2 \nonumber  \\   
		& + \frac{{dL^2\mu}^2}{12}  +  5L^2 \mu^2,
	\end{align}
where  $\tilde{\sigma}^2 = 3 \left(\! 1\!+\! \frac{ c_g d}{b_1 b_2} \!  \right) \sigma_h^2 + \frac{d\sigma_g^2}{b_1 b_2}$, $\tilde{c}_g = 1+ \frac{c_g d}{b_1b_2}$, and $\tilde{c}_h = 1+c_h$.
\end{theorem}
\begin{proof}
Please refer to Appendix \ref{appen_proof_theorem_full}.	
\end{proof}

According to Theorem \ref{theorem_full}, the upper bound of the minimum squared gradient among the global model sequence is composed of four terms. The first term shows that the optimality gap relies on the initial optimality.
The second term shows that the optimality gap depends on the the non-i.i.d. extent of the local data distribution. 
The rest of the terms are related to step size $\mu$ for computing the gradient estimator that is unique in zeroth-order optimization. As pointed out in \cite{nesterov2017random}, we can select an appropriate step size to attain the desired accuracy. 
The following corollary follows by substituting a suitable learning rate $\eta$ and step size $\mu$ into Theorem \ref{theorem_full}.

\begin{corollary}\label{linear_speedup_full}
	Suppose Assumptions \ref{bounded}-\ref{bound_heterogeneity} hold and let $b_1b_2 \leq d$,
	$\mu = (db_1b_2NHT)^{-\frac{1}{4}}$,
	and $\eta = (Nb_1b_2)^{\frac{1}{2}}(dHT)^{-\frac{1}{2}}$, 
	which holds for \eqref{theorem_eta_full} if $T$ is large enough.
	The FedZO algorithm with full device participation satisfies
	\begin{align}\label{rate_full}
		 \min_{t\in [T]} \!  \mathbb{E} \! \left\|\nabla f\left(\bm x^{t}\right)\right\|^2 
		 \leq & \mathcal{O}\left( d^{\frac{1}{2}}{(NHTb_1b_2)}^{-\frac{1}{2}} \right)  \nonumber  \\
 &+ \mathcal{O}\left((db_1b_2NHT)^{-\frac{1}{2}} \right)\!,
	\end{align}
	\noindent where the right hand side of \eqref{rate_full} is dominated by $\mathcal{O}\left( d^{\frac{1}{2}}{(NHTb_1b_2)}^{-\frac{1}{2}} \right)$.
\end{corollary}
We consider the case of $b_1b_2 \leq d$ in Corollary \ref{linear_speedup_full} since the dimension $d$ is generally very large in many ML tasks. Besides, when $b_1b_2 \leq d$, the computational consumption of \eqref{stochastic_estimator} is lower than that of the Kiefer-Wolfowitz type scheme \cite{Anit_random}. 
On the other hand, if $b_1b_2 > d$, $\eta = N^{\frac{1}{2}}(HT)^{-\frac{1}{2}}$, and $\mu = d^{-\frac{1}{2}}(NHT)^{-\frac{1}{4}}$, according to Theorem \ref{theorem_full}, we obtain a convergence rate $\mathcal{O}\left( {(NHT)}^{-\frac{1}{2}} \right)$ for the FedZO algorithm which is independent of dimension $d$. Such a convergence rate is the same as that of the FedAvg algorithm. 
The learning rate of the zeroth-order methods is generally $\sqrt{d}-$times smaller than that of their first-order counterparts \cite{yu2019parallel,nesterov2017random,guanghui}, as the two-point gradient estimator is less accurate than the gradient. In our work, by adopting the mini-batch-type gradient estimator, we can increase the mini-batch sizes $b_1$ and $b_2$ to enhance the accuracy of the gradient estimator and also balance the effects of $d$ and $T$ on the learning rate.

\begin{table}[!t]
	\aboverulesep = 1.2mm
	\belowrulesep = 1.2mm
	\renewcommand{\arraystretch}{1.25}
		\caption{\small Convergence rates of some typical algorithms for stochastic nonconvex unconstrained optimization. 
		}
		\label{tab:main_results}
		\centering\small
		\begin{adjustbox}{width=\columnwidth,center}
			\begin{tabular}{>{\centering\arraybackslash}m{1.8 cm}>{\centering\arraybackslash}m{2.7cm}cc}
				\toprule 
				& \large Algorithm &  \large Convergence rate & \large Maximum value of $H$ \\
				\midrule
				\multirow{2.5}{1.8cm}{\centering \large FL setting}
				& \large FedZO
				& \large $\quad O\!\left(\sqrt{d/NHTb_1b_2}\right)$
				& \large $\min \left\{\mathcal{O}\left((d T)^{\frac{1}{3}}\left(b_1 b_2\right)^{-\frac{1}{3}} N^{-1}\right), \mathcal{O}\left(T N^{-1}\right)\right\}$ \\
				\cmidrule{2-4}
				& \large FedAvg \cite{yu2019parallel}
				& \large $O\!\left(\sqrt{1/NHT}\right)$
				& \large $O\!\left(T^{\frac{1}{3}}N^{-1}\right)$ 
				\\
				\midrule
				\multirow{2.5}{1.8cm}{\centering \large Distributed zeroth-order}
				& \large ZONE-S \cite{zone_hong}
				& \large $O\!\left(d^3/T\right)$
				& --- \\
				\cmidrule{2-4}
				& \large DZOPA \cite{yi2021zeroth}
				& \large $O\!\left(\sqrt{d/NT}\right)$
				& ---
				\\
				\midrule
				\large Centralized zeroth-order & 
				\large ZO-SGD \cite{guanghui}  & 
				\large $O\!\left(\sqrt{d/T}\right)$
				& ---\\
				\bottomrule
				\\[-8pt]
				\multicolumn{4}{l}{\Large Note: 
					\Large $T$ denotes the number of total  communication rounds for FedZO and FedAvg, and
				} \\
				\multicolumn{4}{l}{\Large the number of total iterations for others.} \\[-6pt]
			\end{tabular}
		\end{adjustbox}
	\vspace{-10pt}
\end{table}

\begin{remark}\label{maximum_H}
		In Corollary \ref{linear_speedup_full}, we set the learning rate $\eta = (Nb_1b_2)^{\frac{1}{2}}(dHT)^{-\frac{1}{2}}$, which decreases as the number of local updates (i.e., $H$) increases. In particular, the progress of one local update shrinks by $1/\sqrt{H}$. However, by performing $H$ steps of local updates, we obtain a $\sqrt{H}$-times speedup per communication round. This accords with the derived convergence rate $\mathcal{O}\left( d^{\frac{1}{2}}{(NHTb_1b_2)}^{-\frac{1}{2}} \right)$. According to Corollary \ref{linear_speedup_full}, to reach an $\epsilon$-stationary solution, the FedZO algorithm takes $\mathcal{O}\left(d(Nb_1b_2)^{-1}\epsilon^{-2}\right)$ iterations (i.e., $HT = \mathcal{O}\left(d(Nb_1b_2)^{-1}\epsilon^{-2}\right)$) with learning rate $\eta = (Nb_1b_2)^{\frac{1}{2}}(dHT)^{-\frac{1}{2}}$. If we increase $H$, then the learning rate (i.e., $\eta$) decreases, and we can attain a higher-accuracy solution with the same number of communication rounds (i.e., $T$).
		Besides, when the total iteration number (i.e., $HT$) and the learning rate (i.e., $\eta$) are fixed, we can increase the number of local iterations (i.e., $H$) and reduce the number of communication rounds (i.e., $T$) to reach the same accuracy, which enhances the communication efficiency. It is worth noting that the number of local iterations cannot be arbitrarily large.
		According to \eqref{theorem_eta_full}, to achieve the largest reduction in communication overhead while preserving convergence, the optimal value of $H$ is $\min \left \{\mathcal{O}\left( (dT)^{\frac{1}{3}}(b_1b_2)^{-\frac{1}{3}}N^{-1} \right), \mathcal{O}\left( TN^{-1} \right) \right \}$. 
	\end{remark}


\begin{remark}
From Corollary \ref{linear_speedup_full}, we notice that the proposed FedZO algorithm can attain convergence rate $\mathcal{O}\left( d^{\frac{1}{2}}{(NHTb_1b_2)}^{-\frac{1}{2}} \right)$. In particular, FedZO achieves linear speedup in terms of the number of local iterates and the number of participating edge devices compared with the centralized zeroth-order algorithm (i.e., ZO-SGD) that achieves convergence rate $\mathcal{O}\left(d^{\frac{1}{2}}T^{-\frac{1}{2}}\right)$ \cite{guanghui}. For fair comparison, we consider $b_1=b_2=1$.
To attain the same accuracy, compared to DZOPA with convergence rate $\mathcal{O}\left( d^{\frac{1}{2}}{(NT)}^{-\frac{1}{2}} \right)$ \cite{yi2021zeroth}, the number of communication rounds required by the FedZO algorithm can be reduced by a factor of $H$. Besides, it is worth noting that the convergence rate of the FedZO algorithm depends on the dimension of the model parameter. In particular, the convergence speed of FedZO is $\sqrt{d}$ times slower than that of its first-order counterpart, i.e., FedAvg. Such a degeneration is the same as its centralized counterpart \cite{guanghui}. In addition, for FedAvg, the maximum value of $H$ is $O\!\left(T^{\frac{1}{3}}N^{-1}\right)$ \cite{yu2019parallel}, which is smaller than that of FedZO mentioned in Remark \ref{maximum_H}. This is because the learning rate of FedZO is lower than that of FedAvg, which allows more local updates while preserving convergence.
The detailed comparison between the proposed algorithm and the related algorithms is summarized in Table \ref{tab:main_results}.

\end{remark}

\subsection{Partial Device Participation}\label{part edge device participation}

In this subsection, we show the convergence of the FedZO algorithm  with partial device participation. By bounding the minimum squared gradient among the global model sequence, we characterize the convergence of the FedZO algorithm in the following theorem.

\begin{theorem} \label{theorem_part}
	Suppose Assumptions \ref{bounded}-\ref{bound_heterogeneity} hold  and the learning rate satisfies
	\begin{align}\label{theorem_eta_part}
		\eta \leq \min &\left \{\frac{M}{192\tilde{c}_g \tilde{c}_hL}, \frac{M}{72c_hHL},  \frac{2}{MH^2L}, \frac{1}{3\sqrt{\tilde{c}_g}HL}, \right. \nonumber \\
		&~~ \left.  \frac{1}{3\sqrt{MH^3}L}\right \},
	\end{align}
	the FedZO algorithm with partial device participation satisfies 
	\begin{align}\label{equation_part}
		\min_{t \in [T]} 
		 &\mathbb{E}  \left\|\nabla f\left(\bm x^{t}\right)\right\|^2 
		\leq  
		4\frac{ f\left(\bm x^{0}\right)  - f_* }{ HT \eta} + \eta \frac{32 L}{M}   \tilde{\sigma}^2 \nonumber  \\   
		&\!+\! \eta \frac{36 H L \sigma_h^2}{M} \!+\! \frac{{dL^2\mu}^2}{24}  \!+\!  13 L^2 \mu^2,
	\end{align}
	where $\tilde{c}_g$, $\tilde{c}_h$, and $\tilde{\sigma}^2$ are defined in Theorem \ref{theorem_full}.
\end{theorem}

\begin{proof}
Please refer to Appendix \ref{appen_proof_theorem_part}.	
\end{proof}

By comparing \eqref{equation_part} with \eqref{equation_full}, we notice that the third term in \eqref{equation_part} does not appear in \eqref{equation_full}, which is induced by the randomness of device sampling, while full device participation eliminates this randomness, thereby reducing the optimality gap.

Similarly, the following corollary follows by substituting suitable learning rate $\eta$ and step size $\mu$ into Theorem \ref{theorem_part}.

\begin{corollary}\label{linear_speedup}
Suppose Assumptions \ref{bounded}-\ref{bound_heterogeneity} hold and let $b_1b_2 \leq d$, 
$\mu = (db_1b_2MHT)^{-\frac{1}{4}}$, and $\eta = (Mb_1b_2)^{\frac{1}{2}}(dHT)^{-\frac{1}{2}}$,
which holds for \eqref{theorem_eta_part} if $T$ is large enough.
The FedZO algorithm with partial device participation satisfies
\begin{align}\label{part_order}
	 &\min_{t \in [T]}  \mathbb{E} \left\|\nabla f\left(\bm x^{t}\right)\right\|^2 
	\leq  \mathcal{O}\left( d^{\frac{1}{2}}{(MHTb_1b_2)}^{-\frac{1}{2}} \right) \nonumber \\  
	&+\!\mathcal{O}\left( (b_1b_2H)^{\frac{1}{2}}{(dMT)}^{-\frac{1}{2}} \right)  \!+ \!\mathcal{O}\left((db_1b_2MHT)^{-\frac{1}{2}} \right).
\end{align}
\end{corollary}

According to \eqref{part_order}, to attain a linear speedup in terms of the number of local iterates and participating edge devices, the number of local iterates cannot exceed $\mathcal{O} \!\left(  \!d(b_1b_2)^{-1}) \!\right)$. Combining it with constraint \eqref{theorem_eta_part}, we can derive the largest value of $H$ as 
$
\min \!\left \{ \!\mathcal{O} \!\left( \! (dT)^{\frac{1}{3}}(b_1b_2)^{- \!\frac{1}{3}}M^{-1} \right) \!, \! \mathcal{O} \!\left(  \!TM^{-1}  \!\right)  \!, \! \mathcal{O} \!\left(  \!d(b_1b_2)^{-1}) \!\right) \!\right \} \!.
$

\section{AirComp-Assisted FedZO Algorithm}
\label{aircomp}

In this section, we study the implementation of the proposed FedZO algorithm over wireless networks using AirComp, where the edge devices communicate with the central server via wireless fading channels. 

In each communication round, both the downlink model dissemination phase and the uplink model uploading phase involve wireless transmissions.   
As the central server generally has a much greater transmit power than the edge devices, the downlink model dissemination is assumed to be error-free as in most of the existing studies \cite{zhangyingjun_FL,xiaowen_jsac,zhibin_fl,yonina_fl,Accelerated_cobin} and we focus on the uplink model uploading.



\subsection{Over-the-Air Aggregation}\label{aircomp_schedule}

For the FedZO algorithm, a key observation is that the central server is interested in receiving an average of local model updates of scheduled edge devices rather than each individual one. 
In particular, at the $t$-th round, the central server aims to acquire 
\begin{equation}\label{com_goal}
\bm \Delta^{t} = \frac{1}{|\mathcal{M}_t|} \sum_{i \in \mathcal{M}_t} \bm \Delta_i^{t},
\end{equation}
where $|\mathcal{M}_t|$ denotes cardinality of set $\mathcal{M}_t$.
With conventional OMA schemes, the central server in the $t$-th round first receives the local model update, e.g., $ \bm \Delta_i^{t}$, from each edge device, and then takes an average to obtain the desired global model update, i.e., $\bm \Delta^{t}$. However, these schemes may not be spectrum-efficient as the number of required resource blocks or the communication latency linearly increases with the number of participating edge devices. 
AirComp, as a new non-orthogonal multiple access scheme for scalable transmission, allows all edge devices to concurrently transmit their local model updates and exploits the waveform superposition property to achieve spectrum-efficient model aggregation. The communication resource needed for model uploading using AirComp is independent of the number of participating edge devices.
Hence, we adopt AirComp for the aggregation of local model updates in this paper.

Consider a wireless FL system where all edge devices and the server are equipped with a single antenna.
Over wireless fading channels, the local model updates transmitted by edge devices suffer from detrimental channel distortion, which in turn degenerates the convergence performance of the AirComp-assisted FedZO algorithm. 
We thus set a threshold $h_{\mathrm{min}}$ and choose a subset of edge devices $\mathcal{M}_t = \{i~\! | \!~ |h_i^t| \geq h_{\mathrm{min}}\}$ to participate in the training, where $h_i^{t} \in \mathbb{C}$ represents the channel coefficient between edge device $i$ and the central server in round $t$.  
We assume that $h_i^t$ are i.i.d. across different edge devices and communication rounds \cite{yonina_fl,Accelerated_cobin}. Note that we can treat the adopted device scheduling strategy as uniform sampling analyzed in Section \ref{part edge device participation}.  
With AirComp, the scheduled edge devices concurrently transmit their precoded model updates, e.g., $\alpha_i^t\bm \Delta_i^{t}$, to the central server, where $\alpha_i^t$ is the transmit scalar of edge device $i$ at the $t$-th round.
Note that synchronization is required among distributed edge devices as in \cite{zhangyingjun_FL,xiaowen_jsac,zhibin_fl,yonina_fl,Accelerated_cobin}, which can be realized by sharing a reference-clock across the edge devices \cite{airshare} or utilizing the timing advance technique commonly adopted in 4G long term evolution (LTE) and 5G new radio (NR) \cite{timing}.
 We assume that the model update vector $\bm \Delta_i^{t}$ of dimension $d$ can be transmitted within one transmission block while the channel coefficient is invariant during one transmission block  \cite{zhangyingjun_FL,xiaowen_jsac,zhibin_fl,yonina_fl,Accelerated_cobin}.
Thus, the aggregated signal received at the central server can be expressed as 
\equa{\label{aggregation}
	\bm s^{t} =  \sum_{i \in \mathcal{M}_t} h_i^{t} \alpha_i^t\bm \Delta_i^{t} + \bm n_t,
} 
where $\bm n_t \sim \mathcal{CN}(0,\sigma_w^2\bm I_d)$ represents the additive white Gaussian noise (AWGN) vector at the central server. 

\subsection{Transceiver Design}\label{transceiver}

The transmitted signal at each edge device is subject to an energy constraint during one communication round, i.e., $\snorm{\alpha_i^t \bm \Delta_i^{t}} \leq dP$, where $dP$ is the total energy of each edge device in one communication round. We assume that the channel state information (CSI) is available at both the central server and edge devices as in \cite{zhangyingjun_FL,xiaowen_jsac,zhibin_fl,yonina_fl,Accelerated_cobin}. To meet the energy constraint of each edge device, we set the transmit scalar of device $i$ as
\equa{ \label{precoding}
\alpha_i^t = \frac{h_{\mathrm{min}}}{h_i^{t}}\sqrt{\frac{dP}{\Delta^t_{\rm{max}} }},~ \forall i,
}
where $\Delta^t_{\rm{max}} = \max_{i \in \mathcal{M}_t} \snorm{\bm \Delta_i^{t}} $.
The received signal is thus given by
\equa{\label{com_receive}
\bm s^{t} = \sqrt{\frac{dPh_{\mathrm{min}}^2}{\Delta^t_{\rm{max}}}} \sum_{i \in \mathcal{M}_t}  \bm \Delta_i^{t} + \bm n_t.
}
To recover the desired global model update $\bm \Delta^{t}$ in \eqref{com_goal} from $\bm s^{t}$ in \eqref{com_receive}, the central server  scales $\bm s^{t} $ with a receive scalar $\frac{1}{|\mathcal{M}_t|} \sqrt{\frac{\Delta^t_{\rm{max}}}{dPh_{\mathrm{min}}^2}}$, and obtains a noisy version of the global model update as follows
\equa{\label{receive}
	\bm y^{t} = \bm \Delta^{t} + \tilde{\bm n}_t,
} 
where $\tilde{\bm n}_t \sim  \mathcal{CN}\left(0, \frac{\sigma_w^2 \Delta^t_{\rm{max}}}{|\mathcal{M}_t|^2dPh_{\mathrm{min}}^2}  \bm I_d \right).$ As a result, the global model at the central server is updated as $\bm x^{t+1} = \bm x^{t} + \bm \Delta^{t} + \tilde{\bm n}_t$.
To facilitate the transceiver design, each edge device needs to know the maximum of squared norm of local model updates among the participating edge devices, i.e., $\Delta^t_{\rm{max}}$, and the instantaneous channel coefficient between itself and the central server, i.e., $h_i^t$, which can be obtained via feedback from the central server.
Before uplink model aggregation, the central server collects the squared norm of local model update $\snorm{\bm \Delta_i^{t}}$ from each edge device $i \in \mathcal{M}_t$, and then broadcasts $\Delta^t_{\rm{max}}$ to all edge devices.
Besides, the central server estimates and feeds back the channel coefficients to these corresponding edge devices.
It is worth noting that the communication overhead introduced by the exchange of these scalars is negligible when compared with the transmission of high-dimensional model parameters.

\begin{remark}
Different from most existing studies \cite{yangkai_air,zhangyingjun_FL,xiaowen_jsac,zhibin_fl} that only focus on compensating for channel fading, the adopted transmitter design, i.e., \eqref{precoding}, takes the scale of the model update into account. This ensures that the distortion between the obtained signal and the desired signal, i.e., the scaled receiver noise, is proportional to the maximum of the squared norm of local updates. This distortion diminishes when the local model converges. In other words, with such a transmitter design, the detrimental effect of the noise can be eliminated as the iteration proceeds for a convergent algorithm. 

\end{remark}

\subsection{Convergence Analysis for AirComp-Assisted FedZO}\label{aircomp_analysis}

In the following theorem, we characterize the convergence of the AirComp-assisted FedZO algorithm described in the previous two subsections. 



%
\begin{theorem} \label{theorem_noise}
	Suppose Assumptions \ref{bounded}-\ref{bound_heterogeneity} hold and the learning rate satisfies
	\begin{align}\label{theorem_eta_noise}
		& \eta \leq 
		\min \left \{\!\frac{\tilde{M}}{288 \tilde{c}_g\tilde{c}_h L}, \!\frac{\tilde{M}}{108c_hHL}, \! \frac{3}{2NH^2L},\!
		\frac{1}{3\sqrt{\tilde{c}_g}HL},
		  \right. \nonumber \\
	&~~~~~~~~	\left.  \frac{1}{2\sqrt{3NH^3}L}, \frac{\sqrt{\tilde{M}\gamma}}{L\sqrt{2\tilde{c}_gNH^3}},\! \frac{\tilde{M}^2\gamma}{36 \tilde{c}_g\tilde{c}_h NHL}
		\right \},
	\end{align} 
 where $\tilde{M} = \min \{ |\mathcal{M}_t|, t\in [T] \}$.
	The AirComp-assisted FedZO algorithm satisfies 
	\begin{align}\label{}
	 \min_{t\in [T]} & \mathbb{E} \! \left\|\nabla f\left(\bm x^{t}\right)\right\|^2  \! 
\leq 4\frac{ f\left(\bm x^{0}\right)  - f_* }{ HT \eta} 
  \!+\! \eta \frac{32 L}{\tilde{M}} \hat{C}   \tilde{\sigma}^2 
   \nonumber \\
 &+\! \eta \frac{36 H L \sigma_h^2}{\tilde{M}} \!+\! \hat{C}  \frac{dL^2 \mu^2 }{36 }  
 \!+\!  \left( 12 + \frac{\hat{C}}{9} \right) L^2 \mu^2,
 \end{align}
where $\hat{C} = 1 + \frac{NH}{8 \tilde{M} \gamma}$ and $\gamma = \frac{P h_{\mathrm{min}}^2}{\sigma_w^2}$. $\tilde{c}_g$, $\tilde{c}_h$, and $\tilde{\sigma}^2$ are defined in Theorem \ref{theorem_full}.
\end{theorem}

As can be observed from Theorem \ref{theorem_noise}, the convergence rate depends on $\gamma$, which is the minimum receive SNR. Theorem \ref{theorem_noise} almost reduces to Theorem \ref{theorem_part} when $\gamma$ goes to infinity, i.e., noise-free case. 
Obviously, a smaller value of SNR leads to a  slower convergence speed, which meets our intuition. In the following corollary, we show that a same-order convergence rate as the noise-free case presented in Section \ref{part edge device participation} can be achieved with appropriate receive SNR $\gamma$, learning rate $\eta $, and step size $\mu$. 
\begin{corollary} \label{coro_ota}
Suppose Assumptions \ref{bounded}-\ref{bound_heterogeneity} hold, $8 \tilde{M}\gamma \geq NH$, i.e., the communication quality is good enough, and let $b_1b_2 \leq d$, 
$\mu = (db_1b_2\tilde{M}HT)^{-\frac{1}{4}}$ and $\eta = {(\tilde{M}b_1b_2)}^{\frac{1}{2}}(dHT)^{-\frac{1}{2}}$ that holds for \eqref{theorem_eta_noise}, we have
\begin{align}\label{rate_noise}
	 &\min_{t \in [T]}  \mathbb{E} \left\|\nabla f\left(\bm x^{t}\right)\right\|^2 
	\leq  \mathcal{O}\left( d^{\frac{1}{2}}{(\tilde{M}HTb_1b_2)}^{-\frac{1}{2}} \right) \nonumber \\  
	&+\!\mathcal{O}\left( (b_1b_2H)^{\frac{1}{2}}{(d\tilde{M}T)}^{-\frac{1}{2}} \right)  \!+ \!\mathcal{O}\left((db_1b_2\tilde{M}HT)^{-\frac{1}{2}} \right).
\end{align}
\end{corollary}

\begin{remark}
From Corollary \ref{coro_ota}, it can be observed that the upper bound of the minimum squared gradient among the global model sequence approaches to zero as $T$ goes to infinity, while that of the existing algorithms with AirComp is only shown to be bounded by a non-diminishing optimality gap \cite{zhibin_fl,zhangyingjun_FL,xiaowen_jsac}. 
Moreover, the transceiver design in \cite{zhibin_fl,zhangyingjun_FL,xiaowen_jsac} 
is transformed to an optimization problem aiming to minimize this gap, which is computationally expensive. 
In contrast, our transceiver design follows the principle of COTAF \cite{yonina_fl} and mitigates the detrimental impact of channel fading and receiver noise perturbation without the need of solving optimization problems. 
Note that the analysis in \cite{yonina_fl} concentrates on the first-order algorithm under the strongly convex setup and relies on the assumption that the second-order moment of the stochastic gradient is bounded by a constant, which is restrictive \cite{tighter_theory} and not required in this paper.
\end{remark}

%
%

\section{Simulation Results}\label{simulation}

In this section, we present simulation results to evaluate the effectiveness of the proposed FedZO algorithm for applications of federated black-box attack and softmax regression.

\subsection{Federated Black-Box Attack}
\label{blackbox_results}

The robustness of machine learning (ML) models is an important performance metric for their practical application. For example, in an image classification model, the prediction results of the ML model are expected to be the same as the decision that humans make.
In other words, the same output should be generated by a robust model if the input image is perturbed by a noise imperceptible to human. 
To evaluate the robustness of ML models, black-box attacks can be adopted, where the adversary acts as a standard user that does not have access to the inner structure of ML models and can only query the outputs (label or confidence score) for different inputs. 
This situation occurs when attacking ML cloud services where the model only serves as an API. Due to the black-box property, the optimization of black-box attacks falls into the category of zeroth-order optimization.

We consider federated black-box attacks \cite{yi2021zeroth} on the image classification DNN models that are well trained on some standard datasets. Federated black-box attacks aim to collaboratively generate a common perturbation such that the perturbed images are visually imperceptible to a human but could mislead the classifier. For image $\bm z_i$, the attack loss \cite{carlini2017towards} is given by
\begin{align}
	\psi_{i}(\bm x)= &\max \left\{\Phi_{y_{i}}\left(\frac{1}{2} \tanh \left(\tanh ^{-1} 2 \bm z_{i}+ \bm x\right)\right)  \right . \nonumber \\
	& -\left. \max_{j \neq y_{i}}\left\{\Phi_{j}\left(\frac{1}{2} \tanh \left(\tanh ^{-1} 2 \bm z_{i}+ \bm x\right)\right)\right\}, 0\right\} \nonumber \\
	&+c\left\|\frac{1}{2} \tanh \left(\tanh ^{-1} 2 \bm z_{i} + \bm x\right)- \bm z_{i}\right\|^{2}, 
\end{align}
where $y_i$ denotes the label of image $\bm z_i$, $\Phi_j(\bm z)$ represents the prediction confidence of image $\bm z$ to class $j$, $\frac{1}{2} \tanh \left(\tanh ^{-1} 2 \bm z_{i} + \bm x\right)$ is the adversarial example of $\bm z_i$, $\frac{1}{2} \tanh \left(\tanh ^{-1} 2 \bm z_{i} + \bm x\right)- \bm z_{i}$ is the distortion perturbed by $\bm x$ in the original image space. The first term of $\psi_{i}(\bm x)$ measures the probability of failing to attack. The last term of $\psi_{i}(\bm x)$ represents the distortion induced by $\bm x$ in the original image space. 
The goal of attack is to find a visually small perturbation to mislead the classifier $\Phi(\cdot)$ that can be realized by minimizing $\psi_{i}(\bm x)$. 
Parameter $c$ balances the trade-off between the adversarial success and distortion loss. We denote the dataset at edge device $n$ as $\mathcal{D}_n$. The attack loss of device $n$ can be expressed as $f_n(\bm x) =  \frac{1}{|\mathcal{D}_n|}\sum_{i\in \mathcal{D}_n} \psi_{i}(x)$. Federated black-box attacks of a DNN model can be formulated as:
$
\min_{\bm x\in \mathbb{R}^d} \frac{1}{N} \sum_{n=1}^{N} f_n(\bm x),
$ 
which can be tackled by the proposed FedZO algorithm.

In this experiment setting, all edge devices share one well-trained DNN classifier\footnote{\text{https://github.com/carlini/nn\_robust\_attacks}} that has a testing accuracy of $82.3 \%$ on CIFAR-10 dataset \cite{carlini2017towards}.
We pick $4992$ correctly classified samples from the training set of image class ``deer" (containing 5,000 samples) and then distribute these samples to edge devices without overlapping. Each edge device is assigned a random number of samples.

We set the balancing parameter $c=1$.
The mini-batch sizes are set to $b_1 = 25$ and $b_2 = 20$.
The learning rate and step size are set to $\eta = 0.001$ and $\mu = 0.001$, respectively.

In Fig. \ref{fig:h}, we show the impact of the number of local updates on the convergence performance of the proposed FedZO algorithm with full device participation. 
Specifically, we vary the number of local updates $H \in \{5, 10, 20, 50\}$ and present the attack loss versus the number of communication rounds. It can be observed that the FedZO algorithm can effectively reduce the attack loss for different values of $H$. Besides, as $H$ increases, the convergence speed of the FedZO algorithm tends to increase. This demonstrates the speedup in the number of the local iterates as shown in Section \ref{theory}. We further compare the performance of the proposed FedZO algorithm with DZOPA \cite{yi2021zeroth} and ZONE-S \cite{zone_hong}. For DZOPA, the learning rate (i.e., $\eta$) and step size (i.e., $\mu$) are set to $0.005$ and $0.001$, respectively. For ZONE-S, the penalty parameter (i.e., $\rho$, defined in \cite{zone_hong}) and step size (i.e., $\mu$) are set to $500$ and $0.001$, respectively.
Note that DZOPA was proposed for the peer-to-peer architecture which cannot be directly applied to our considered server-client architecture.
For comparison, we depict the performance of DZOPA under a fully-connected graph.
For fairness, we also upgrade the two-point stochastic gradient estimator of \cite{yi2021zeroth} to a mini-batch-type one as in \eqref{stochastic_estimator_ori}. Results show that the FedZO algorithm outperforms the baselines even when $H=5$.
With a larger number of local updates, the attack loss of the FedZO algorithm decreases much faster than that of the baselines.



\begin{figure*}
	\centering
	\begin{minipage}{.3\textwidth}
		\begin{subfigure}{\textwidth}
			\centering
			\includegraphics[width=2.4in]{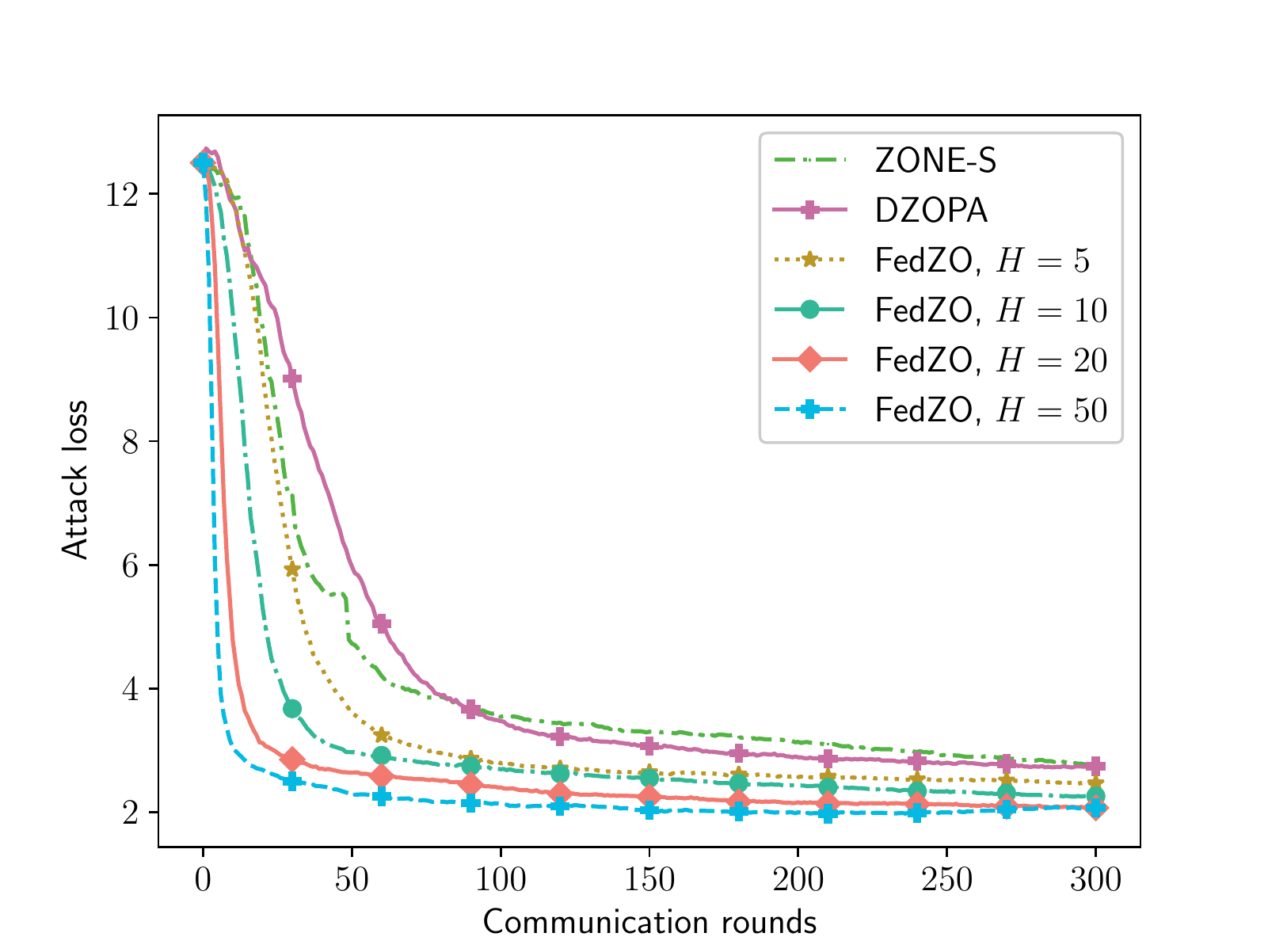}
			\caption{Impact of number of local updates when $N=10$ and $M = 10$.}\label{fig:h}
			\vspace{0.02cm}
		\end{subfigure}\\ 
	\end{minipage}
	\hspace*{\fill} 
	\begin{minipage}{.3\textwidth}
		\begin{subfigure}{\textwidth}
			\centering
			\includegraphics[width=2.4in]{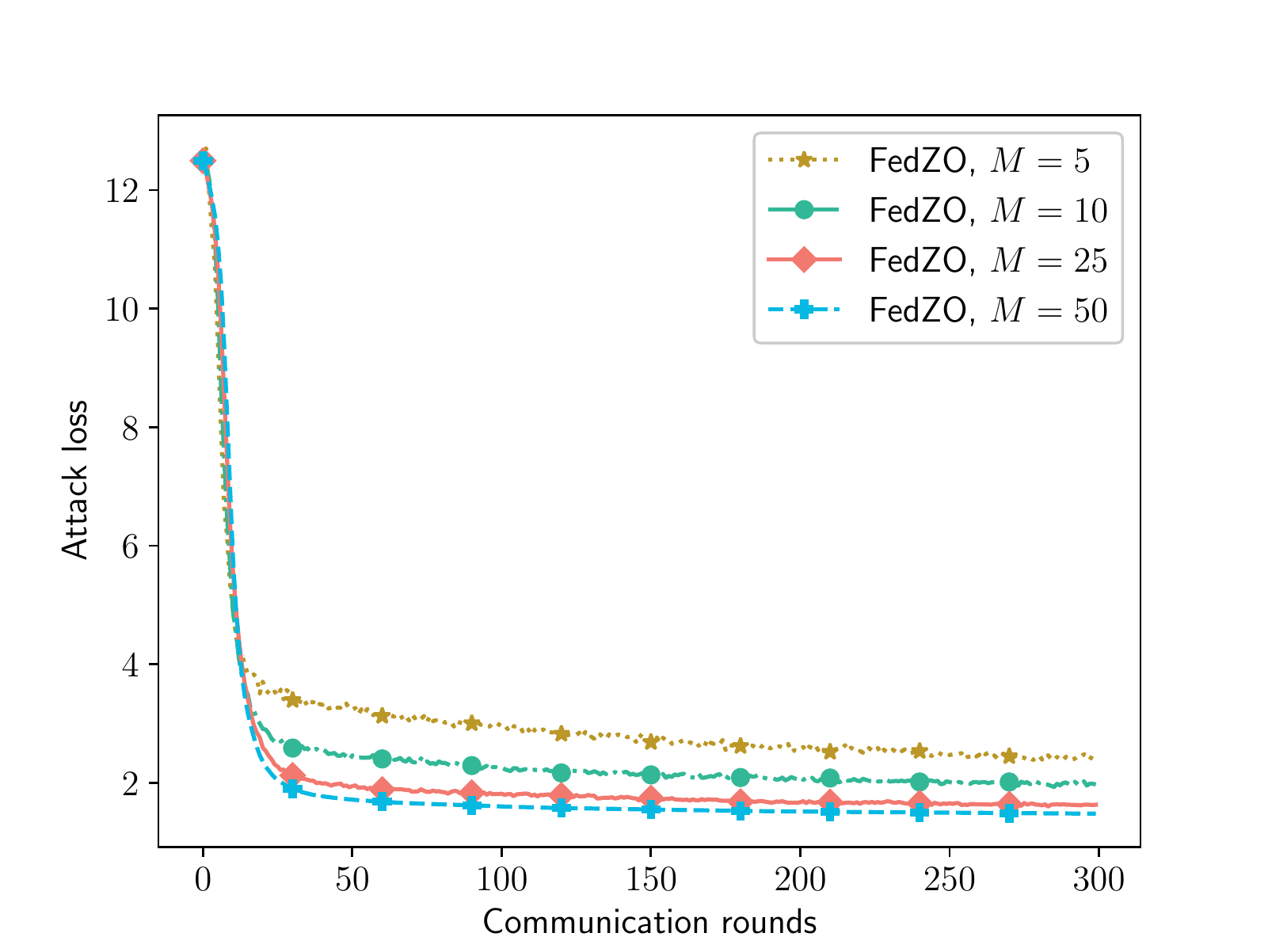}
			\caption{Impact of number of participating edge devices when $N=50$ and  $H=20$.}\label{fig:n}
			\vspace{0.03cm}
		\end{subfigure}%
	\end{minipage}
	\hspace*{\fill}
	\begin{minipage}{.3\textwidth}
		\begin{subfigure}{\textwidth}
			\centering
			\includegraphics[width=2.4in]{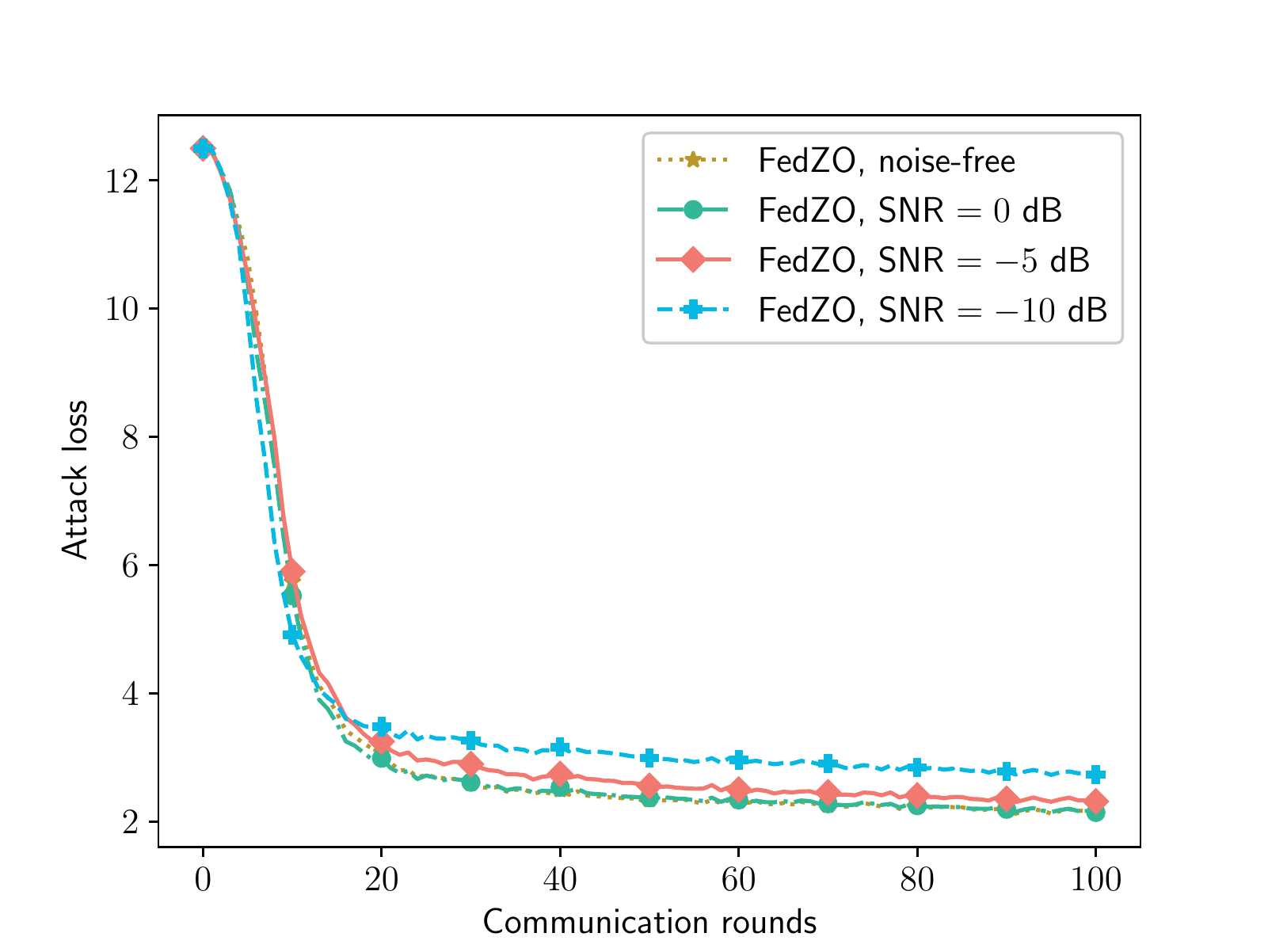}
			\caption{Impact of SNR when $N=50$ and  $H=20$.}\label{fig:snr}
			\vspace{0.02cm}
		\end{subfigure}%
	\end{minipage}%
	\caption{Attack loss of the federated black-box attack on CIFAR-10 dataset.}
	\vspace{-0.4cm}
\end{figure*}
\begin{figure}
	\vspace{-0.4cm}
	\centering
	\includegraphics[width=3in]{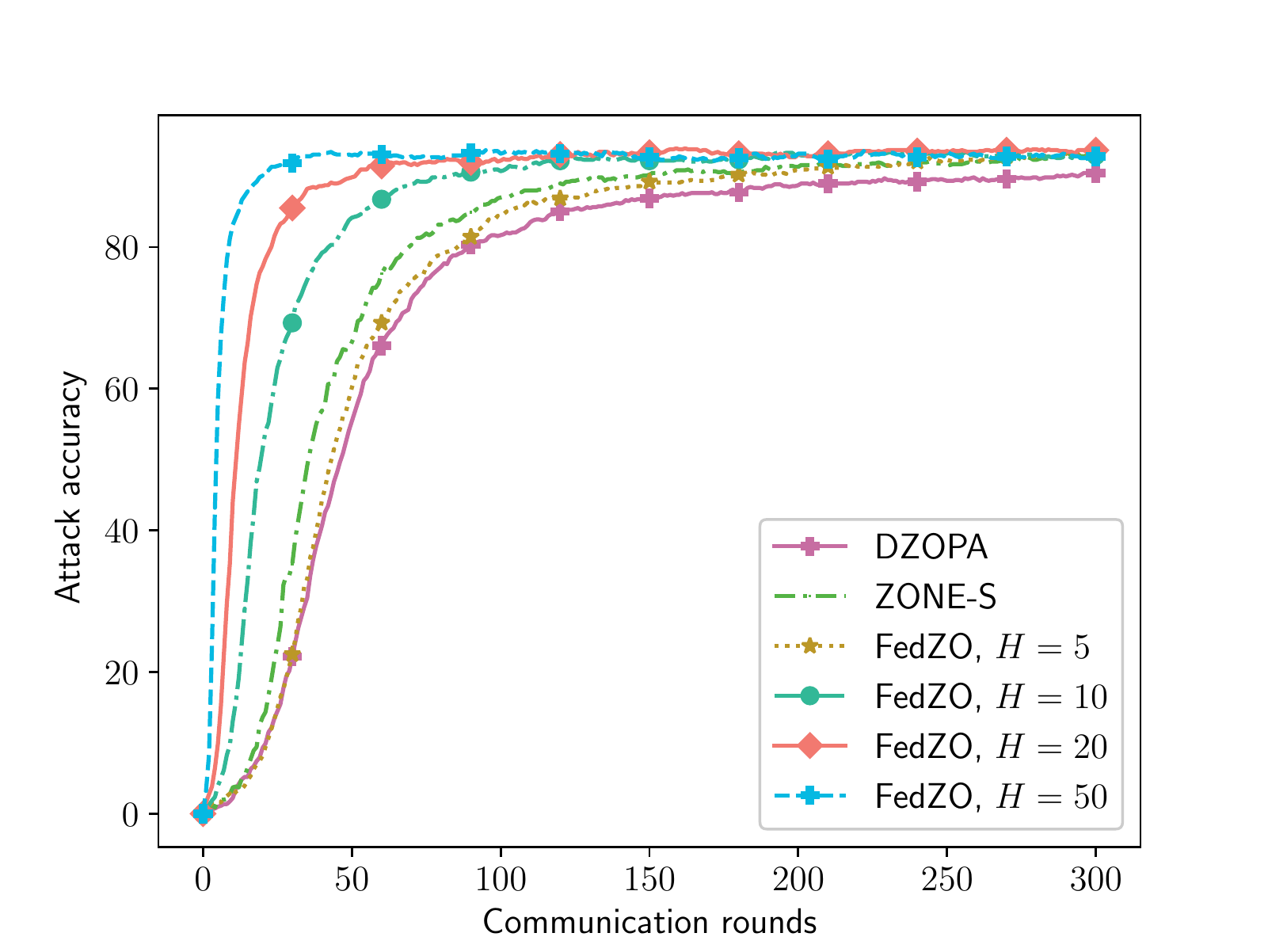}
	\caption{Attack accuracy of the federated black-box attack on CIFAR-10 dataset.}\label{fig:acc}
	\vspace{-0.5cm}
\end{figure}

\begin{figure*}
	\centering
	\begin{minipage}{.3\textwidth}
		\begin{subfigure}{\textwidth}
			\centering
			\includegraphics[width=2.4in]{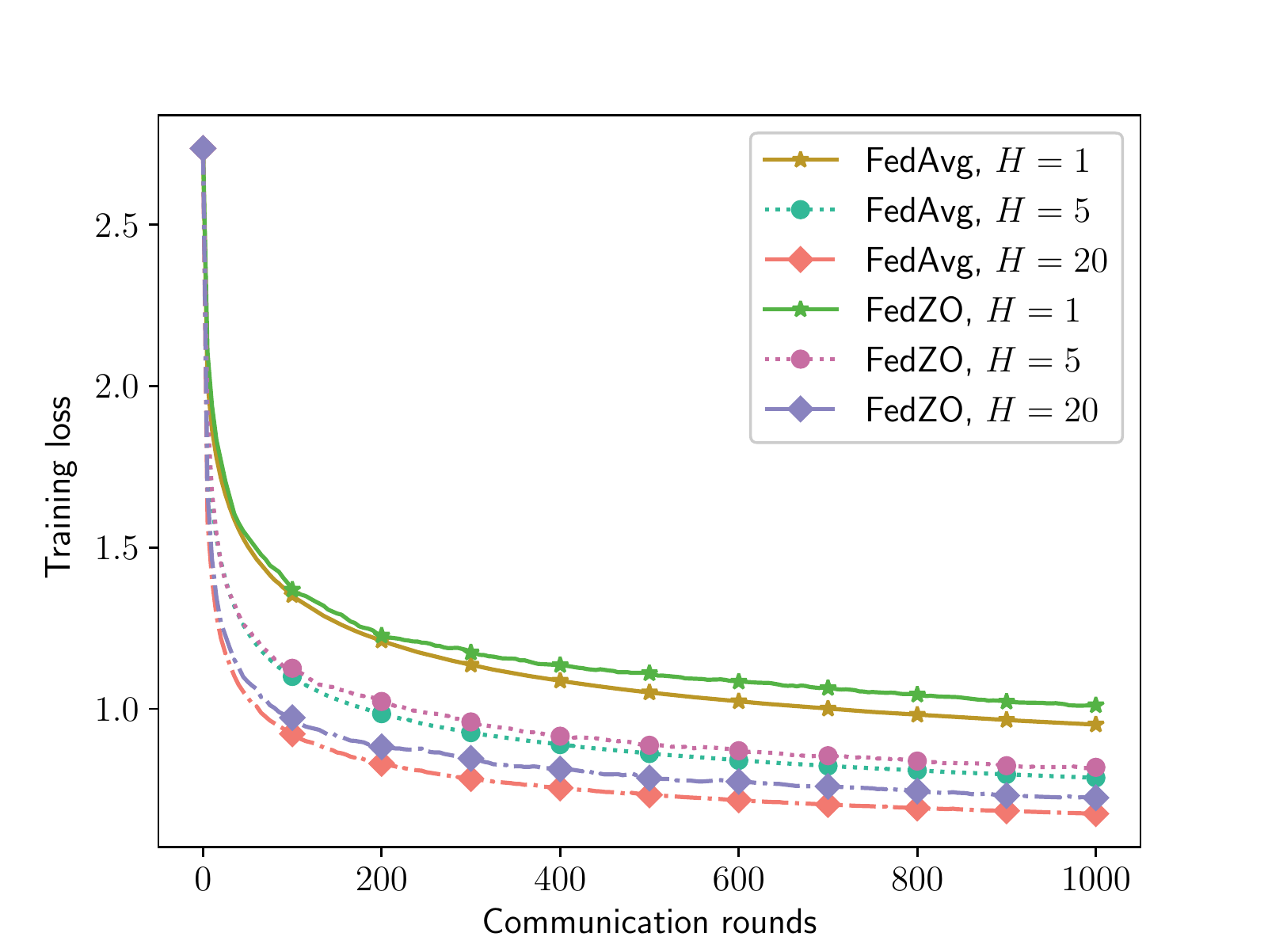}
			\caption{Impact of $H$ on the training loss.}
		\end{subfigure}\\
		\begin{subfigure}{\textwidth}
			\centering
			\includegraphics[width=2.4in]{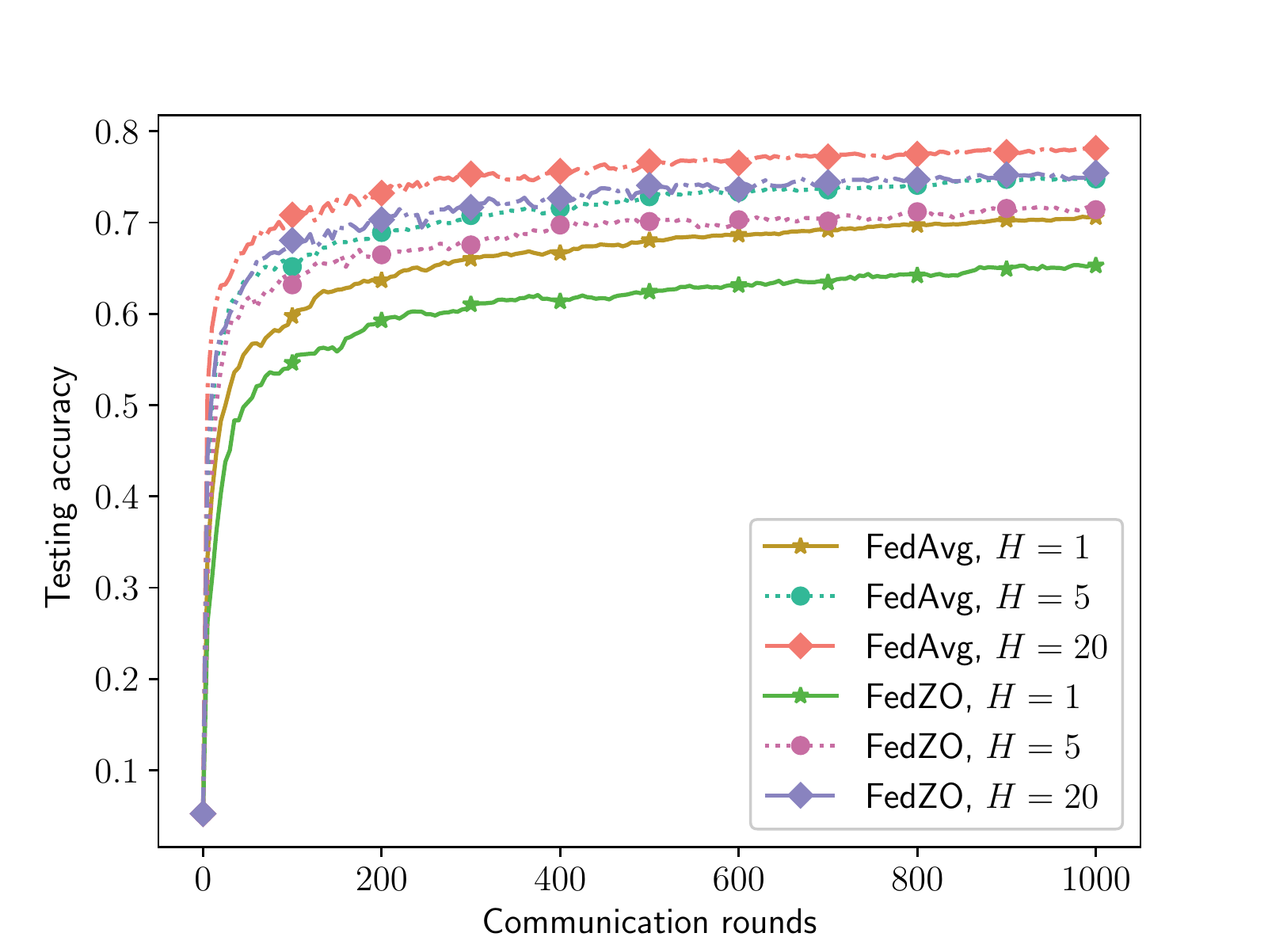}
			\caption{Impact of $H$ on the testing accuracy.}
		\end{subfigure}%
		\caption{The convergence results
				on the softmax regression problem with Fashion-MNIST dataset when $N=50$ and $M=20$.}
		\label{mnist_H}
	\end{minipage}
	\hfill
	\begin{minipage}{.3\textwidth}
		\begin{subfigure}{\textwidth}
			\centering
			\includegraphics[width=2.4in]{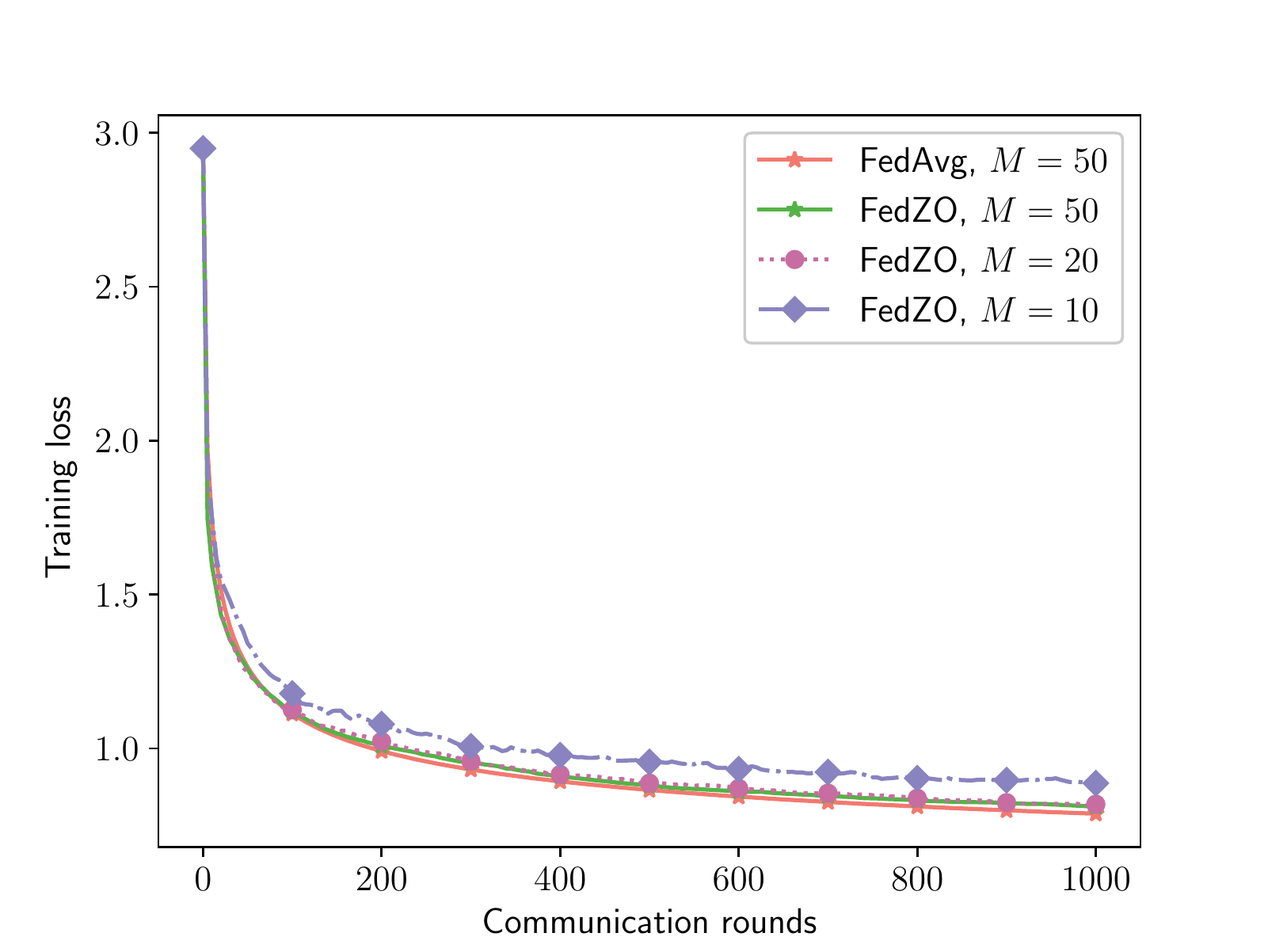}
			\caption{Impact of $M$ on the training loss.}
		\end{subfigure}\\
		\begin{subfigure}{\textwidth}
			\centering
			\includegraphics[width=2.4in]{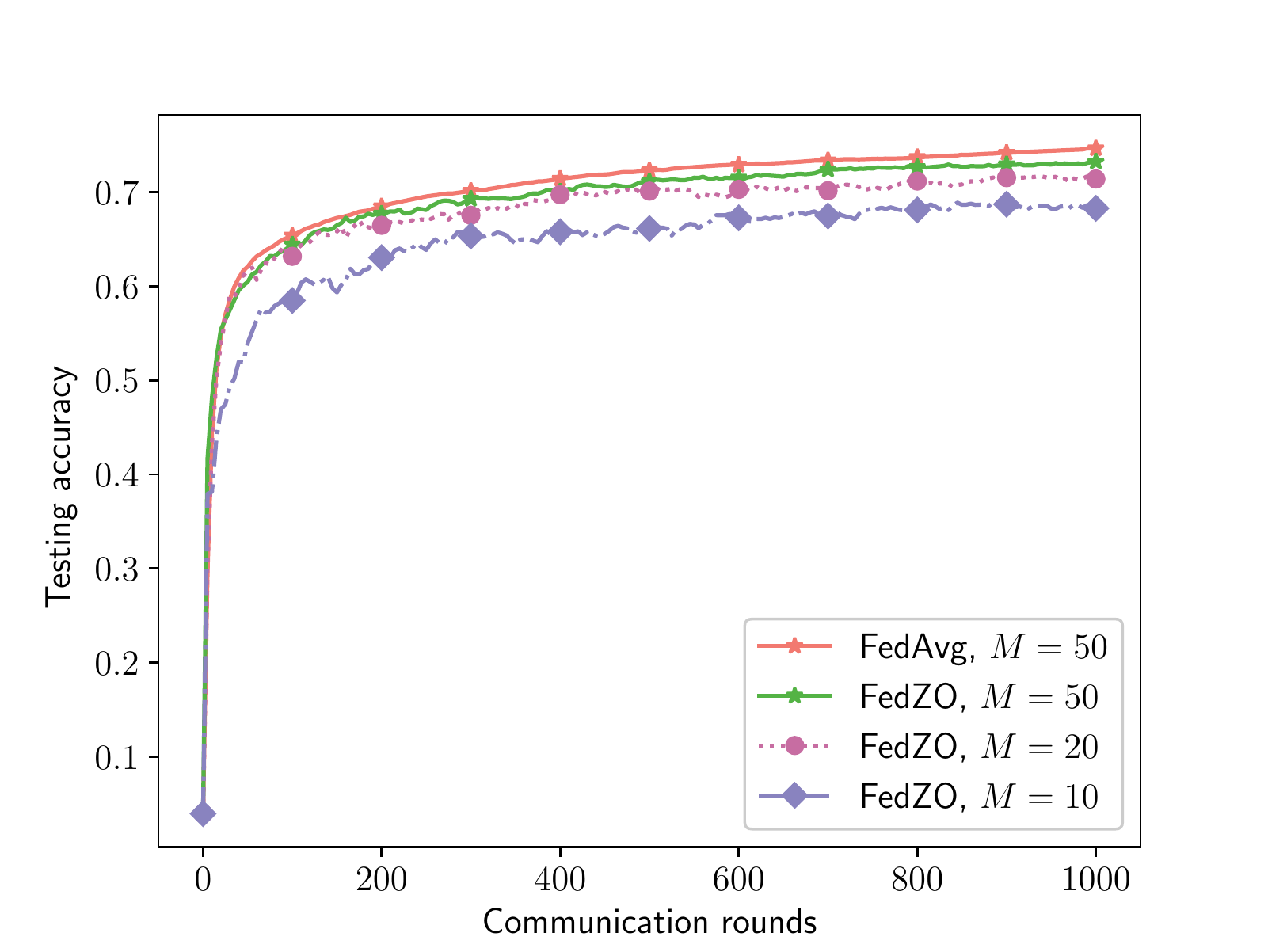}
			\caption{Impact of $M$ on the testing accuracy.}
		\end{subfigure}%
		\caption{The convergence results
				on the softmax regression problem with Fashion-MNIST dataset when $N=50$ and $H=5$.}
		\label{mnist_N}
	\end{minipage}
	\hfill
	\begin{minipage}{.3\textwidth}
		\begin{subfigure}{\textwidth}
			\centering
			\includegraphics[width=2.4in]{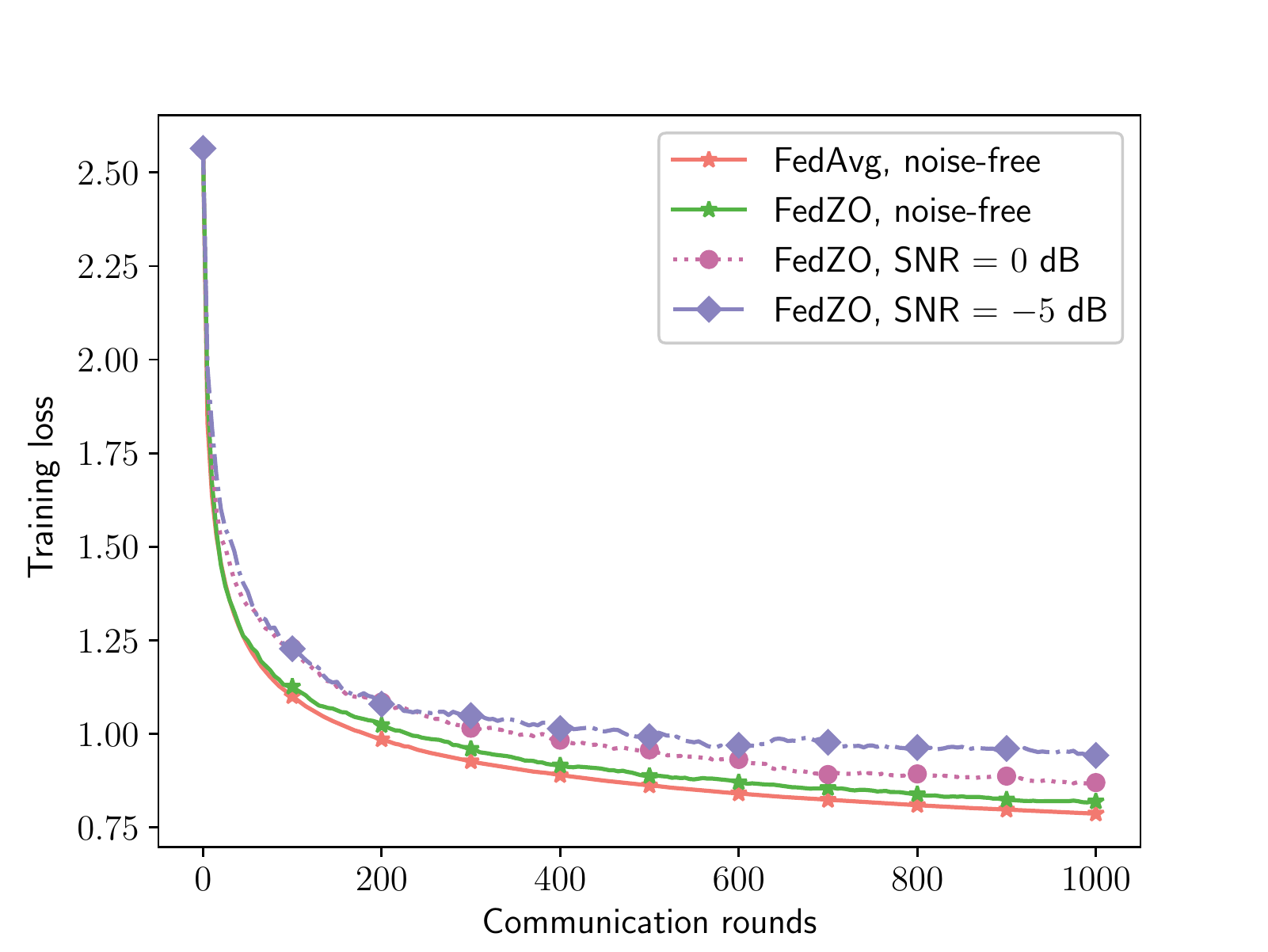}
			\caption{Impact of SNR on the training loss.}
		\end{subfigure}\\
		\begin{subfigure}{\textwidth}
			\centering
			\includegraphics[width=2.4in]{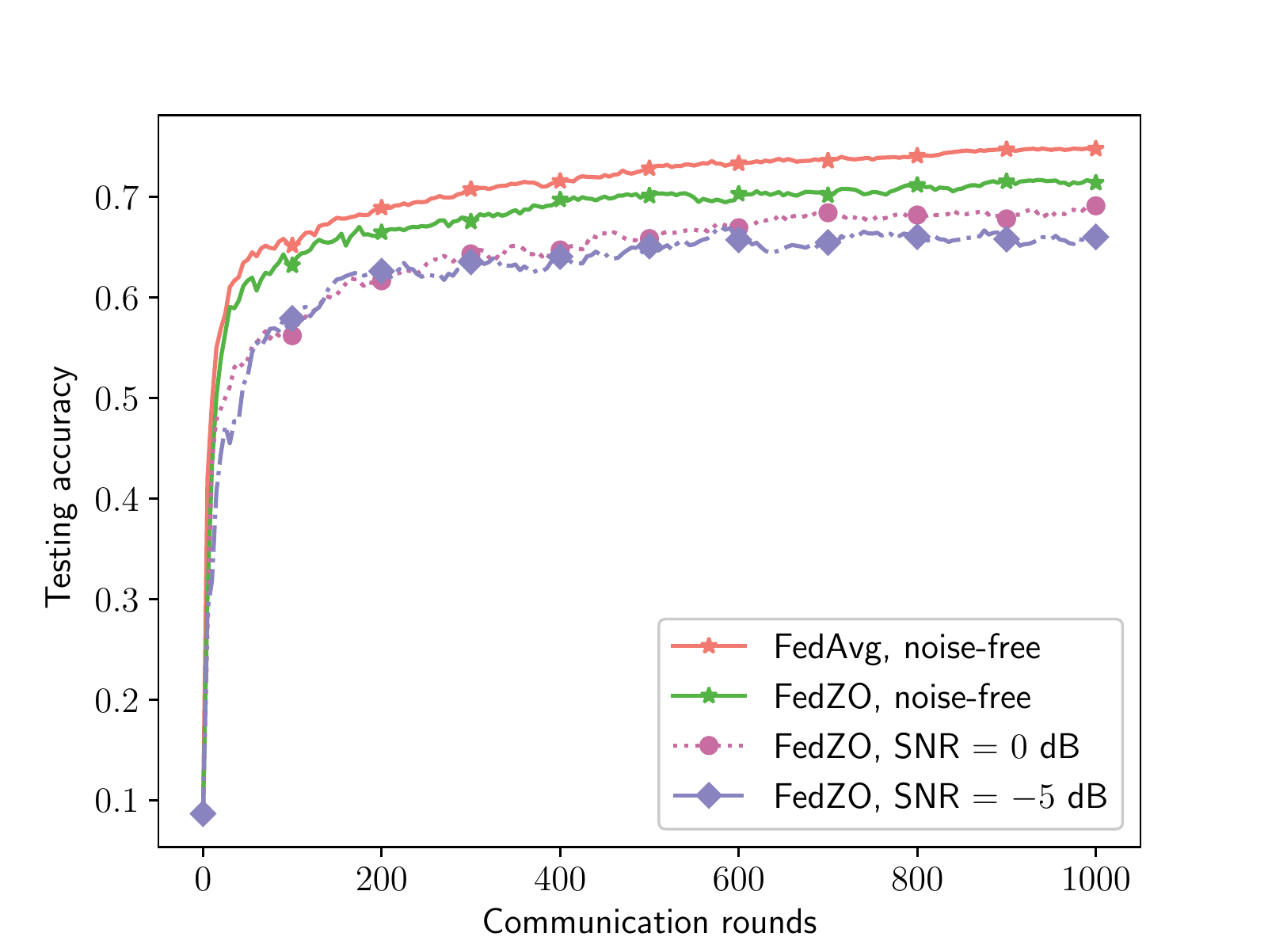}
			\caption{Impact of SNR on the testing accuracy}
		\end{subfigure}%
		\caption{The convergence results
				on the softmax regression problem with Fashion-MNIST dataset when $N=50$ and $H=5$.}
		\label{mnist_SNR}
	\end{minipage}  
	\vspace{-0.4cm}
\end{figure*}

Fig. \ref{fig:n} shows the convergence performance of the FedZO algorithm versus the number of participating edge devices when $H=20$. 
The number of participating edge devices $M$ takes values from set $\{5, 10, 25, 50\}$. 
It is clear that our proposed algorithm works well in terms of reducing attack loss under the four different values of $M$.
As can be observed, increasing the number of edge devices gives rise to a better convergence speed. 
We can observe the speedup in the number of participating devices, which matches well with our analysis in Section \ref{theory}.

In Fig. \ref{fig:snr}, we show the performance of the AirComp-assisted FedZO algorithm presented in Section \ref{aircomp}. Without loss of generality, we model the channel as $h_i^t \sim \mathcal{CN}(0,1), ~\forall i,t$, and set the threshold $h_{\mathrm{min}} = 0.8$. Besides, we set the number of local iterates to $H=20$.
We take the FedZO algorithm with noise-free aggregation as benchmark with $H=20$, where the participating edge devices are the same as that of the case with noise.
We plot the attack loss versus the number of communication rounds under different SNR, i.e., $P/\sigma_w^2 \in \{-10 ~\text{dB}, ~-5~\text{dB}, ~0 ~\text{dB}\}$.
As can be observed, the convergence of the FedZO algorithm can be preserved under such SNR settings. 
Besides, increasing the SNR accelerates the convergence speed of the FedZO algorithm.  
Especially, the FedZO algorithm with noise $\text{SNR}=0 ~\text{dB}$ attains a comparable performance with the noise-free case. 
These observations are in line with our theoretical analysis in Section \ref{aircomp_analysis}. 

Fig. \ref{fig:acc} further demonstrates the superiority of the proposed FedZO algorithm over the baselines in terms of the attack accuracy. As can be observed, the proposed FedZO algorithm achieves a better attack accuracy than ZONE-S and DZOPA when $H$ is greater than 10. 
	When $H = 5$, ZONE-S achieves a higher attack accuracy than FedZO at the cost of incurring a higher attack loss.
	This is because the perturbation generated by ZONE-S brings a large distortion.

\subsection{Softmax Regression}


We further validate our algorithm on the task of softmax regression, which corresponds to the multinomial classifier. We compare the FedZO algorithm with the FedAvg algorithm \cite{mcmahan17a}, which is the most representative first-order method. We set the learning rate (i.e., $\eta$), step size (i.e., $\mu$), and mini-batch sizes (i.e., $b_1$ and $b_2$) for the FedZO algorithm as $0.001$, $0.001$, $25$, and $20$, respectively. For the FedAvg algorithm, we set the learning rate (i.e., $\eta$) as $0.001$. 

We apply the softmax regression model to a 10-class-classification task on Fashion-MNIST \cite{fmnist} dataset. 
Through out the experiment, we set the number of devices $N=50$. 
Our strategy for constructing the non-i.i.d. data distribution follows the seminal work \cite{mcmahan17a}.
In particular, we sort the samples in the training set according to their labels, and then divide the training set into 100 shards of size 600.  
We then assign two shards to each device, such that each device owns a dataset of $1,200$ samples. Each edge device is assigned with at most four distinctive image labels.


As shown in Fig. \ref{mnist_H}, the convergence speed of the FedZO algorithm is slightly slower than that of the FedAvg algorithm under the same number of local updates.
This gap is brought by the uncertainty of the gradient estimator that FedZO utilizes.  Further, we notice that the FedZO algorithm with $H=20$ achieves comparable performance as the FedAvg algorithm with $H=5$.
As the FedZO algorithm only relies on the zeroth-order information, the slightly decreased performance is reasonable, and demonstrates the effectiveness of the FedZO algorithm.
This also shows that the FedZO algorithm can serve as a satisfactory alternative for the FedAvg algorithm when the first-order information is not available.

In Fig. \ref{mnist_N}, we take the FedAvg algorithm as a benchmark when $H=5$ and $M=50$, and further investigate the impact of the number of participating edge devices, i.e., $M$, on the convergence behaviour of the FedZO algorithm.
The phenomenon of  speedup in $M$ can be witnessed in both the training loss and testing accuracy. 
We observe that the FedZO algorithm with $H=5$ and $M=50$ attains a comparable performance with the FedAvg algorithm. 

Fig. \ref{mnist_SNR} shows the performance of the AirComp-assisted FedZO algorithm over wireless networks with the same channel setting as mentioned in Section \ref{blackbox_results}. 
It can be observed that our proposed algorithm converges as the number of communication rounds increases and performs well when the SNR is not very small, e.g., $\text{SNR} \in \{-5~\text{dB}, ~0 ~\text{dB} \}$. Results also show that a greater SNR leads to a higher convergence speed. This result fits well with our analysis.



\section{Conclusion}\label{conclusion}

In this paper, we developed a derivative-free FedZO algorithm to handle federated optimization problems without using the gradient or Hessian information. Under non-convex settings, we characterized its convergence rate on non-i.i.d. data, and demonstrated the linear speedup in terms of the number of participating devices and local iterates. Subsequently, we established the convergence guarantee for the AirComp-assisted FedZO algorithm to support the implementation of the proposed algorithm over wireless networks. 
Simulation results demonstrated the effectiveness of the proposed FedZO algorithm and showed that the FedZO algorithm could serve as a satisfactory alternative for the FedAvg algorithm. It was also validated that the AirComp-assisted FedZO algorithm could attain a comparable performance with that of the noise-free case under certain SNR conditions.



\appendix

To prove Theorems \ref{theorem_full} and  \ref{theorem_part}, we first characterize per round progress by Lemmas \ref{progress_lem_full} and \ref{progress_lem_part}, respectively, and then bound the client drift during $H$ local iterates by Lemma \ref{drift_lem}. To prove Theorem \ref{theorem_noise}, we further bound the wireless noise by Lemma \ref{drift_lem_H}.
The proofs of these lemmas are deferred to Appendix \ref{proof_lemmas}. Before presenting the proofs, we first  introduce some notations that are frequently used in this appendix. 


Let $\mathcal{F}^{(t,k)}$ be a $\sigma$-field representing all the historical information of the FedZO algorithm up to the start of the $k$-th iteration of the $t$-th round. $\mathbb{E}_t$ and $\mathbb{E}_{t}^k$  denote expectations conditioning on $\mathcal{F}^{(t,0)}$ and  $\mathcal{F}^{(t,k)}$, respectively. Let 
\begin{small}
	$
	\zeta_t = \big \{\{\xi_{i,m}^{(t,k)}, \bm v_{i,n}^{(t,k)} \}_{i=1,2,\ldots,N; k=0,1,\ldots,H-1 }^{m=1,2,\ldots,b_1; n = 1,2,\ldots,b_2} \big \}~\text{and}~
	 \zeta_t^{k} = \big \{\{\xi_{i,m}^{(t,\tau)}, \bm v_{i,n}^{(t,\tau)} \}_{i=1,2,\ldots,N; \tau=0,1,\ldots,k }^{m=1,2,\ldots,b_1; n = 1,2,\ldots,b_2} \big\}.
	 $
\end{small}
 $\mathbb{E}_{\zeta_t}$, $\mathbb{E}_{\zeta_t^{k}}$, and $\mathbb{E}_{\mathcal{M}_{t}}$  denote expectations over $\zeta_t$, $\zeta_t^{k}$, and $\mathcal{M}_{t}$, respectively.

\subsection{Proof of Theorem \ref{theorem_full}}\label{appen_proof_theorem_full}
For notational ease, we denote $\delta_t = \mathbb{E}_{\zeta_t} \!\left[\frac{1}{N} \! \sum_{i=1}^N \! \sum_{k=0}^{H-1} \!  \|\bm x_i^{(t,k)} \!-\! \bm x^t\|^2 \!\right]$.
Before proving Theorem \ref{theorem_full}, we present the following two lemmas. The first lemma characterizes how the global loss, i.e., $f(\bm x^t)$, evolves as the iteration continues. 
\begin{lemma}\label{progress_lem_full}
With Assumptions \ref{bounded}-\ref{bound_heterogeneity} and full device participation, 
by letting $\eta \leq \frac{1}{2HL}$, we have
\begin{align}\label{progress_full_equation}
&\mathbb{E}_{t} \left[ f\left(\bm x^{t+1}\right) \right] 
\leq f\left(\bm x^t\right)  
\!-\! \left( \frac{\eta H}{2} \!-\! \eta^2 \frac{6\tilde{c}_g\tilde{c}_hHL}{N} \right) \left\|\nabla f\left(\bm x^t\right)\right\|^2   \nonumber \\
&~~ \!+\!  \left( \eta L^2 \!+\! \eta^2 \frac{6\tilde{c}_gL^3 }{N} \right) \delta_t \!+\! \eta^2 \frac{2HL}{N} \tilde{\sigma}^2   \nonumber \\
&~~ \!+\eta^2\frac{2HL^3\mu^2}{N}\!+\! \eta^2 \frac{d^2 H L^3\mu^2}{2Nb_1b_2} \!+\! \eta H L^2 \mu^2.
\end{align}
Please refer to Appendix \ref{appen_proof_lem_1} for the proof.
\end{lemma}

Lemma \ref{progress_lem_full} implies that we need to bound $\delta_t$, which is tackled by the following lemma.
\begin{lemma}
\label{drift_lem}
With Assumptions \ref{bounded}-\ref{bound_heterogeneity} and $\eta \leq \frac{1}{3HL\sqrt{\tilde{c}_g}}$, we have 
\begin{small}
\equa{\label{}
\delta_t
\!\leq \! 3 \eta^2 \tilde{c}_g \tilde{c}_h H^3 \left \|\!\nabla f(\bm x^t) \!\right \|^2
\!+\! H^3 L^2 \eta^2\mu^2 \!+\! \eta^2 H^3 \tilde{\sigma}^2 
\!+\! \frac{d^2 H^3 L^2}{4b_1b_2} \eta^2 \mu^2. \nonumber 
}
\end{small}
Please refer to Appendix \ref{appen_proof_lem_2} for the proof. 
\end{lemma}
By substituting the upper bound of $\delta_t$ in Lemma \ref{drift_lem} into \eqref{progress_full_equation}, we obtain 
\begin{small}
\equa{\label{}
&\mathbb{E}_t\!\left[ \!f\left(\bm x^{t+1}\right) \!\right]
\! \leq 
f\left(\bm x^t\right)  
\!-\! \Big(\! \frac{\eta H}{2} \!-\! 3\eta^2 \tilde{c}_g \tilde{c}_h Q_1\!(\eta) 
  \Big) \! \left\|\nabla f(\bm x^t) \right\|^2   \\
&+  \eta^2 \tilde{\sigma}^2 
Q_1(\eta)   +Q_1(\eta)  \big(L^2 \eta^2\mu^2 
+ \frac{d^2 L^2}{4b_1b_2} \eta^2 \mu^2\big) + \eta H L^2 \mu^2, \nonumber
}
\vspace{-2mm}
\end{small}

\noindent where $Q_1(\eta) = \frac{2HL}{N} + (\eta L^2 + \eta^2 \frac{6\tilde{c}_g L^3 }{N})H^3$. Under condition \eqref{theorem_eta_full}, we have
\begin{small}
\[
Q_1(\eta) \leq \frac{6HL}{N}, ~\eta^2 \frac{18\tilde{c}_g \tilde{c}_hHL}{N} \leq \frac{\eta H}{4}, ~ L^2\eta^2\mu^2 \leq \frac{ NL\eta\mu^2 }{72},
\]
\vspace{-2mm}
\end{small}

\noindent 
and $\frac{d^2 L^2 \eta^2 \mu^2}{4b_1b_2} \leq \frac{d^2NL \eta \mu^2 }{288b_1b_2\tilde{c}_g\tilde{c}_h} $. Recalling the definition of $\tilde{c}_g$ and the fact that $c_g\tilde{c}_h \geq 1$, we have $\frac{d^2}{b_1b_2\tilde{c}_g\tilde{c}_h}\leq \frac{d}{ c_g\tilde{c}_h}\leq d$. With the above inequalities, we have
\begin{small}
\begin{align}\label{befor_expectation_full}
\mathbb{E}_t \left[ f\left(\bm x^{t+1}\right) \right] 
\leq 
&f\left(\bm x^t\right)  
- \frac{\eta H}{4} \left\|\nabla f\left(\bm x^t\right)\right\|^2 +\eta^2 \frac{6 HL}{N} \tilde{\sigma}^2 \nonumber \\
&  +\frac{\eta dHL^2 \mu^2 }{48}  + \frac{13}{12}\eta H L^2 \mu^2.
\end{align}
\end{small}
By taking expectation on both sides of \eqref{befor_expectation_full} and telescoping from $t=0$ to $T-1$, we obtain
\begin{small}
\begin{align}\label{}
 \frac{1}{T} \sum_{t=0}^{T-1}   \mathbb{E} \left\|\nabla f\left(\bm x^{t}\right)\right\|^2 
 \leq 
 & 4\frac{f\left(\bm x^{0}\right)  - \mathbb{E} \left[ f\left(\bm x^{T}\right)  \right]}{ HT \eta} + \eta \frac{24 L}{N}  \tilde{\sigma}^2   \nonumber  \\
 & + \frac{dL^2 \mu^2 }{12}  +  5L^2 \mu^2.
\end{align}
\end{small}
By Assumption \ref{bounded}, 
 i.e., $f(\bm x) \geq f_*$,
we obtain Theorem \ref{theorem_full}.

\subsection{Proof of Theorem \ref{theorem_part}}\label{appen_proof_theorem_part}
We first characterize how the global loss, i.e., $f(\bm x^t)$, evolves as the iteration continues in the following lemma. 
\begin{lemma}\label{progress_lem_part}
With Assumptions \ref{bounded}-\ref{bound_heterogeneity} and partial device participation, by letting the learning rate $\eta \leq \frac{1}{2HL}$, we have
\equa{\label{progress_part_equation}
&\mathbb{E}_t \left[ f\left(\bm x^{t+1}\right) \right] \\
&\leq f\left(\bm x^t\right)  
\!-\! \left( \frac{\eta H}{2} - \eta^2 \frac{6\tilde{c}_g\tilde{c}_hHL}{M}  - \eta^2\frac{9H^2Lc_h}{M} \right) \left\|\nabla f\left(\bm x^t\right)\right\|^2   \\
& +\! \left ( \!\eta L^2 \!+\! \eta^2 \frac{6\tilde{c}_gL^3 }{M} \!+\! \eta^2 18 H L^3\right) \delta_t \!+\! \eta^2\frac{2HL}{M} \tilde{\sigma}^2 \!+\! \frac{9 \eta^2 H^2L \sigma_h^2}{M}\\
&+\! \eta^2\frac{2HL^3\mu^2}{M} \!+\! \eta^2 \frac{d^2 H L^3\mu^2}{2Mb_1b_2} \!+\! 6 \eta^2 H^2L^3\mu^2   \!+\! \eta H L^2 \mu^2.  \nonumber
}
Please refer to Appendix \ref{appen_proof_lem_3} for the proof.
\end{lemma}

\noindent By combining Lemmas \ref{drift_lem} and \ref{progress_lem_part}, we have
\begin{small}
\equa{\label{}
&\mathbb{E}_t\left[ \!f\left(\bm x^{t+1}\right) \!\right]\\
 \leq & f\left(\bm x^t\right) 
\!-\! \Big(\! \frac{\eta H}{2} \!-\! \eta^2\frac{9H^2Lc_h}{M}  \!-\! 3 \tilde{c}_g\tilde{c}_h \eta^2 Q_2\!(\eta) \! \Big ) \! \left\|\nabla f(\bm x^t) \right\|^2   \\
&+\! \eta^2 \tilde{\sigma}^2 Q_2(\eta)  \! +\! \frac{9 \eta^2 H^2L \sigma_h^2}{M} \! +\! Q_2(\eta) \Big( L^2 \eta^2\mu^2
+ \frac{d^2  L^2}{4b_1b_2} \eta^2 \mu^2 \Big)  \\
& 
+\! 6 \eta^2 H^2L^3\mu^2   
\!+\! \eta H L^2 \mu^2, \nonumber
}
\vspace{-2mm}
\end{small} 

\noindent where $Q_2(\eta) = \frac{2HL}{M} + \big(\eta L^2 + \eta^2 \frac{6\tilde{c}_gL^3 }{M} + \eta^2 18 H L^3\big)H^3$. Under condition \eqref{theorem_eta_part}, we have 
\begin{small}
\[
\eta^2\frac{9H^2Lc_h}{M} \leq \frac{\eta H}{8}, ~ Q_2(\eta)\leq \frac{8HL}{M},~ \frac{24 \eta^2 \tilde{c}_g\tilde{c}_h HL}{M}  \leq \frac{\eta H}{8},
\]
\end{small}
\begin{small}
\[
L^2\eta^2\mu^2 \leq \frac{ ML\eta\mu^2 }{192}, ~ \frac{d^2 L^2 \eta^2 \mu^2}{4b_1b_2} \leq \frac{d^2ML \eta \mu^2 }{768b_1b_2\tilde{c}_g\tilde{c}_h},
\]
\vspace{-2mm}
\end{small}

\noindent and $6 \eta^2 H^2L^3\mu^2 \leq 2\eta HL^2\mu^2$.
Recalling the definition of $\tilde{c}_g$ and the fact that $c_g\tilde{c}_h \geq 1$, we have $\frac{d^2}{b_1b_2\tilde{c}_g\tilde{c}_h}\leq \frac{d}{ c_g\tilde{c}_h}\leq d$. With the above inequalities, we have 
\begin{small}
\equa{
& \mathbb{E}_t \! \left[\! f\left(\bm x^{t+1}\right) \! \right] 
\! \leq \!
f\left(\bm x^t\right)  
\!-\! \frac{\eta H}{4} \left\|\nabla f\left(\bm x^t\right)\right\|^2   \!+ \!
\eta^2 \frac{8 HL}{M}  \tilde{\sigma}^2 \\
& + \frac{9 \eta^2 H^2L \sigma_h^2}{M} 
+ \frac{\eta H dL^2 {\mu}^2}{96}  
+ \Big(\!3\!+\!\frac{1}{24}\!\Big)\eta H L^2 \mu^2. \nonumber
}
\end{small}

By following the similar steps as in Appendix \ref{appen_proof_theorem_full}, we obtain Theorem \ref{theorem_part}.

\subsection{Proof of Theorem \ref{theorem_noise}}


By denoting $\tilde{\bm x}^{t+1} = \bm x^t + \frac{1}{|\mathcal{M}_t|} \sum_{i\in \mathcal{M}_t} \Delta_i^t$ as the noise-free aggregated model, we have $\bm x^{t+1} = \tilde{\bm x}^{t+1}  + \tilde{\bm n}_t$. 
By applying the smoothness of $f(\bm x)$ in Assumption \ref{assump_smooth}, we obtain
\equa{\label{iter_noise}
f\left(\bm x^{t+1}\right) \leq f\left(\tilde{\bm x}^{t+1}\right) + \left \langle \nabla f\left(\tilde{\bm x}^{t+1}\right),  \tilde{\bm n}_t \right \rangle  + \frac{L}{2} \snorm{\tilde{\bm n}_t}.
}
We denote
$s^{(t,H)} = \frac{1}{N} \sum_{i=1}^N \mathbb{E}_{\zeta_t} 
\|\bm x_i^{(t,H)} \!-\! \bm x^t\|^2$.
By taking an expectation for \eqref{iter_noise} conditioning on $\mathcal{F}^{(t,0)}$ and utilizing $\mathbb{E}_{\zeta_t}\big [\max_i \|\bm x_i^{(t,H)} \!-\! \bm x^t\|^2\big] \leq N s^{(t,H)} $, we have
\equa{\label{evolu_air}
\mathbb{E}_t\! \left[ f\left(\bm x^{t+1}\right) \right]  
 \leq \mathbb{E}_t\! \left[f\left(\tilde{\bm x}^{t+1}\right)  \right]   
\!+\! \frac{L}{2\gamma} \frac{N}{|\mathcal{M}_t|^2}s^{(t,H)}.
}
The above result suggests that we need to bound $s^{(t,H)}$, which can be handled by the following lemma. 
\begin{lemma}
\label{drift_lem_H}
With Assumptions \ref{bounded}-\ref{bound_heterogeneity} hold, we have
\begin{align}\label{H_step_2_all}
 s^{(t,H)} \leq & 6\tilde{c}_g H L^2 \eta^2  \delta_t + 6\tilde{c}_g \tilde{c}_h H^2 \eta^2  \left \|\nabla f(\bm x^t) \right \|^2 \!+\! 2H^2 \eta^2 \tilde{\sigma}^2 \nonumber \\
 & +\! 2 H^2L^2 \eta^2\mu^2 \!+\! \frac{d^2 H^2 L^2 }{2b_1b_2} \eta^2 \mu^2.
 \end{align}
Please refer to Appendix \ref{appen_proof_lem_4} for the proof.	
\end{lemma}
By combining Lemmas \ref{drift_lem}, \ref{progress_lem_part}, and \ref{drift_lem_H}, we obtain 
\begin{small}
\equa{\label{}
&\mathbb{E}_t\! \left[ \!f\left(\bm x^{t+1}\right) \!\right] \\
&\leq f\left(\bm x^t\right)  
\!-\! \Big( \frac{\eta H}{2} \!-\! \eta^2\frac{9H^2Lc_h}{|\mathcal{M}_t|} \!-\! 3\tilde{c}_g\tilde{c}_h\eta^2 \tilde{Q}_3(\eta) \Big) \left\|\nabla f\left(\bm x^t\right)\right\|^2 \\ 
&  +\! \eta^2 \tilde{\sigma}^2  \tilde{Q}_3(\eta) \! +\! \frac{9 \eta^2 H^2L \sigma_h^2}{|\mathcal{M}_t|} \!+\!  \tilde{Q}_3(\eta) \Big( L^2 \eta^2\mu^2
+ \frac{d^2  L^2}{4b_1b_2} \eta^2 \mu^2 \Big) \\
&+ 6 \eta^2 H^2L^3\mu^2   + \eta H L^2 \mu^2, \nonumber 
}
\end{small}
\hspace{-0.2cm} where $\tilde{Q}_3(\eta) \!=\!  \frac{NH^2L }{|\mathcal{M}_t|^2\gamma} \!+\! Q_3(\eta) $ and
\begin{small}
	\[
	Q_3(\eta) \!=\! \frac{2HL}{|\mathcal{M}_t|} \!+\! \big( \eta L^2 \!+\!  \eta^2 \frac{6\tilde{c}_g L^3 }{|\mathcal{M}_t|} \!+\! \eta^2 18 H L^3 \!+\! \frac{3\tilde{c}_gNHL^3 \eta^2}{|\mathcal{M}_t|^2\gamma} \big)
	H^3.
	\]
\end{small}
Under condition \eqref{theorem_eta_noise}, we have
\begin{small}
\[
\eta^2\frac{9H^2Lc_h}{|\mathcal{M}_t|} \leq \frac{\eta H}{12},~  \frac{3\tilde{c}_g\tilde{c}_hNH^2L \eta^2}{|\mathcal{M}_t|^2\gamma}  \leq \frac{\eta H}{12},~ Q_3(\eta) \leq \frac{8HL}{|\mathcal{M}_t|},
\]
\end{small}
\begin{small}
\[
\eta^2 \frac{24 \tilde{c}_g\tilde{c}_hHL}{|\mathcal{M}_t|}  \leq \frac{\eta H}{12}, ~ 6 \eta^2 H^2L^3\mu^2 \leq 2\eta HL^2\mu^2, 
\]
\end{small}
\begin{small}
\[
L^2\eta^2\mu^2 \leq \frac{ |\mathcal{M}_t|L\eta\mu^2 }{288}, \text{and}~ \frac{d^2 L^2 \eta^2 \mu^2}{4b_1b_2} \leq \frac{d^2 |\mathcal{M}_t| L \eta \mu^2 }{1152 b_1b_2\tilde{c}_g\tilde{c}_h}.
\]
\end{small}
\hspace{-0.15cm}Recalling the definition of $\tilde{c}_g$ and the fact that $c_g\tilde{c}_h \geq 1$, we have $\frac{d^2}{b_1b_2\tilde{c}_g\tilde{c}_h}\leq \frac{d}{ c_g\tilde{c}_h}\leq d$. With the above inequalities, we have 
\begin{small}
	\begin{align}\label{aqaa}
&\mathbb{E}_t \!\left[\!\left(\bm x^{t+1}\right) \!\right] \leq \! f\left(\bm x^t\right)  
\!-\! \!\frac{\eta H}{4}  \! \left\|\nabla f\left(\bm x^t\right)\right\|^2 \!+\! \eta^2 \frac{8HL}{|\mathcal{M}_t|} \hat{C} \tilde{\sigma}^2 \nonumber \\
&  +\! \frac{9 \eta^2 H^2L \sigma_h^2}{|\mathcal{M}_t|}\!+\! \hat{C}  \frac{\eta dHL^2 \mu^2 }{144} \!+\!  \left(3 + \frac{\hat{C}}{36} \right)\eta H L^2 \mu^2,
\end{align}
\end{small}

\noindent where $\hat{C} = 1 + \frac{NH}{8 \tilde{M} \gamma}$ and $\tilde{M} \leq |\mathcal{M}_t|, ~ \forall t$.
After reorganizing \eqref{aqaa}, we obtain
\begin{small}
\begin{align}\label{befor_expectation_noise} 
 & 
 \left\|\nabla \!f\left(\bm x^t\right)\right\|^2 
\leq 
4\frac{f\left(\bm x^t\right)  
   \! - \! \mathbb{E}_t \left[ \!f\left(\bm x^{t+1}\right) \!\right]}{\eta H}   
  \!+\! \eta \frac{32 L}{\tilde{M}}   \hat{C} \tilde{\sigma}^2 \nonumber \\
  &\!+\! \frac{36 \eta H L \sigma_h^2}{\tilde{M}} 
 \!+\! \hat{C}  \frac{dL^2 \mu^2 }{36}  
 \!+\!  \left(12 + \frac{\hat{C}}{9} \right) L^2 \mu^2.
\end{align}
\end{small}
By following the similar derivation as in Appendix \ref{appen_proof_theorem_full}, we obtain Theorem \ref{theorem_noise}.

\subsection{Proof of Lemmas}\label{proof_lemmas} 

As Lemma \ref{progress_lem_full} is a simplified version of  Lemma \ref{progress_lem_part}, we first prove Lemma \ref{progress_lem_part} and then prove Lemma \ref{progress_lem_full}.

\subsubsection{Proof of Lemma \ref{progress_lem_part}}\label{appen_proof_lem_3}
Based on Assumption \ref{assump_smooth} and
\begin{small}
	$\eta \frac{1}{M} \sum_{i \in \mathcal{M}_{t}} \sum_{k=0}^{H-1} \bm e_i^{(t,k)} = \bm x^{t+1} - \bm x^t$,
\end{small}
we have
\begin{small}
\begin{align}\label{descent_lemma}
f\left(\bm x^{t+1}\right) 
\leq & f\left(\bm x^t\right) 
- \eta  \left \langle \nabla f\left(\bm x^t\right), 
 \frac{1}{M} \sum_{i \in \mathcal{M}_{t}} \sum_{k=0}^{H-1} \bm e_i^{(t,k)} \right \rangle \nonumber \\
& + \eta^2\frac{L}{2} \left \|\frac{1}{M} \sum_{i \in \mathcal{M}_{t}} \sum_{k=0}^{H-1} \bm e_i^{(t,k)} \right \|^2.
\end{align}
\end{small}
By taking an expectation for \eqref{descent_lemma} 
conditioning on $\mathcal{F}^{(t,0)}$, we obtain
\begin{small}
\begin{align}\label{descent_lemma_1}
\mathbb{E}_t \! \left[ \!f\left(\bm x^{t+1}\right) \!\right] 
 \leq & f\left(\bm x^t\right) 
\underbrace{- \eta  \mathbb{E}_{\zeta_t} \!\left[ \!\left \langle \! \nabla f\left(\bm x^t\right) \!,\! \frac{1}{M} \mathbb{E}_{\mathcal{M}_{t}} \! \sum_{i \in \mathcal{M}_{t}} \sum_{k=0}^{H-1}\! \bm e_i^{(t,k)} \!\right \rangle \!\right]}_{T_1} \nonumber \\
&+ \eta^2\frac{L}{2} \underbrace{\mathbb{E}_t \left \| \frac{1}{M} \sum_{i \in \mathcal{M}_{t}} \sum_{k=0}^{H-1} \bm e_i^{(t,k)} \right\|^2 }_{T_2}.
\end{align}
\end{small}
As $\mathcal{M}_{t}$ is uniformly sampled from $N$ edge devices, by utilizing \cite[Lemma 4]{LiHYWZ20}, we have
\begin{small}
\[
T_1 = - \eta  \mathbb{E}_{\zeta_t} \left[  \left \langle \nabla f\left(\bm x^t\right) , \frac{1}{N} \sum_{i=1}^N \sum_{k=0}^{H-1} \bm e_i^{(t,k)} \right \rangle \right].
\]
\end{small}
\hspace{-0.15cm}According to \eqref{mini_batch_equal}, we have
\begin{small}
\equa{\label{two_exp}
\mathbb{E}_{\zeta_t} \left[
\frac{1}{N} \sum_{i=1}^N \sum_{k=0}^{H-1} \left(\bm e_i^{(t,k)} 
- \nabla f_i^{\mu}\left(\bm x_i^{(t,k)} \right)\right) \right] = 0.
}
\end{small}
With the above two equalities, we have 
\begin{small}
	\[
T_1 = - \eta \mathbb{E}_{\zeta_t} \left[ \left \langle \nabla f\left(\bm x^t\right) , \frac{1}{N}\sum_{i=1}^N \sum_{k=0}^{H-1} \nabla f_i^{\mu}\left(\bm x_i^{(t,k)} \right)  \right \rangle \right].
\]
\end{small}
Because of the equality $2\langle \bm a, \bm b \rangle \!=\! \|\bm a\|^2 \!+\! \|\bm b\|^2 \!-\! \|\bm a \!-\! \bm b\|^2$, we obtain
\begin{small}
\begin{align}\label{inne_gra_app}
&T_1 \!=\! -\frac{\eta H}{2} \left\|\nabla f\left(\bm x^t\right)\right\|^2 
\!-\! \frac{\eta H}{2} \mathbb{E}_{\zeta_t}\! \left\|\!\frac{1}{NH}\sum_{i=1}^N \! \sum_{k=0}^{H-1} \nabla f_i^{\mu}\left(\bm x_i^{(t,k)} \right)\!\right\|^2  \nonumber\\
&+ \frac{\eta H}{2} \underbrace{ \mathbb{E}_{\zeta_t}\! \left\|\!\frac{1}{NH}\sum_{i=1}^N \sum_{k=0}^{H-1} \left(\nabla f_i^{\mu}\left(\bm x_i^{(t,k)} \right) - \nabla f_i\left(\bm x^t \right)\right)\!\right\|^2 }_{T_3}.
\end{align}
\end{small}
For $T_3$, we have
\begin{small}
\begin{align}\label{diff_app_grad}
&T_3\leq \!  \frac{1}{NH} \mathbb{E}_{\zeta_t}\! \left[\! \sum_{i=1}^N \sum_{k=0}^{H-1} \left \|\! \nabla f_i^{\mu}\left(\bm x_i^{(t,k)} \right) - \nabla f_i\left(\bm x^t \right) \right \|^2 \!\right] \nonumber \\
&= \!  \frac{1}{NH}\mathbb{E}_{\zeta_t}\! \left[\! \sum_{i=1}^N \! \sum_{k=0}^{H-1} \! \left \| \nabla f_i^{\mu}\left(\bm x_i^{(t,k)} \right) \!\mp \! \nabla f_i\left(\bm x_i^{(t,k)} \right) 
 \!-\! \nabla f_i\left(\bm x^t \right) \right \|^2 \! \right]\ \nonumber \\
&\leq \! \frac{2}{NH}  \mathbb{E}_{\zeta_t}\! \left[\!\sum_{i=1}^N \sum_{k=0}^{H-1}\! \left \| \nabla f_i^{\mu}\left(\bm x_i^{(t,k)} \right) - \nabla f_i\left(\bm x_i^{(t,k)} \right) \right \|^2 \!\right] \nonumber \\
&~~~+ \frac{2}{NH}\mathbb{E}_{\zeta_t}\!\left[\! \sum_{i=1}^N \sum_{k=0}^{H-1} \! \left \|\nabla f_i\left(\bm x_i^{(t,k)} \right) - \nabla f_i\left(\bm x^t \right) \right \|^2 \!\right] \nonumber \\
&\leq 2 L^2 \mu^2 +  \frac{2L^2}{NH}\mathbb{E}_{\zeta_t}\! \left[\! \sum_{i=1}^N \sum_{k=0}^{H-1} \! \left\|\bm x_i^{(t,k)} - \bm x^t\right\|^2 \!\right],\, 
\end{align}
\vspace{-0.3cm}
\end{small}

 \noindent where the first inequality follows by the Jensen's inequality, $a\mp b$ represents $a\!-\!b\!+\!b$, the second inequality holds because of the Cauchy-Schwartz inequality, and the last inequality follows by \cite[Lemma 2]{yi2021zeroth} and the smoothness of $f_i\left(\bm x \right)$ in Assumption \ref{assump_smooth}.
By substituting \eqref{diff_app_grad} into \eqref{inne_gra_app}, we have

\begin{small}
\vspace{-0.3cm}
\begin{align}\label{T_1_end}
&T_1 \leq -\frac{\eta H}{2} \left\|\nabla f\left(\bm x^t\right)\right\|^2 
- \frac{\eta H}{2} \mathbb{E}_{\zeta_t} \left \|\frac{1}{NH}\sum_{i=1}^N \sum_{k=0}^{H-1} \nabla f_i^{\mu}(\bm x_i^{(t,k)} ) \right\|^2  \nonumber \\
&+ \eta H L^2 \mu^2 + \eta L^2 \mathbb{E}_{\zeta_t}\! \left[\!\frac{1}{N}\! \sum_{i=1}^N \sum_{k}^H \!  \|\bm x_i^{(t,k)} - \bm x^t\|^2 \!\right].
\end{align}
\end{small}
For $T_2$, according to the Cauchy-Schwartz inequality, we have
\begin{small}
\begin{align}\label{T_2_mid}
T_2 \leq & 2 \underbrace{\mathbb{E}_t \left \| \frac{1}{M} \sum_{i \in \mathcal{M}_{t}} \sum_{k=0}^{H-1} 
\left( \bm e_i^{(t,k)} - \nabla f^{\mu}_i\left(\bm x_i^{(t,k)} \right)  \right) \right\|^2 }_{T_4} \nonumber \\ 
&+ 2 \underbrace{ \mathbb{E}_t \left \| \frac{1}{M} \sum_{i \in \mathcal{M}_{t}} \sum_{k=0}^{H-1} 
\nabla f^{\mu}_i\left(\bm x_i^{(t,k)} \right) \right\|^2 }_{T_5}.  
\end{align}
\end{small}

\noindent By denoting $\bm h_i = \sum_{k=0}^{H-1} 
( \bm e_i^{(t,k)} \!-\! \nabla f^{\mu}_i(\bm x_i^{(t,k)} )  )$ and utilizing \eqref{mini_batch_equal},
we have $\mathbb E_{\zeta_t} [\bm h_i] = 0$. Due to the independence between $\bm h_i$ and $\bm h_j$, $\forall j \neq i$,
 we have $\mathbb E_{\zeta_t} \left[\left \langle \bm h_i, \bm h_j  \right \rangle \right] = 0$.  We thus obtain
 
\begin{small}
\begin{align}
T_4 &=\frac{1}{M^2}  \mathbb{E}_{\mathcal{M}_t} \left[\! \sum_{i \in \mathcal{M}_{t}} \mathbb{E}_{\zeta_t} \left \|\sum_{k=0}^{H-1} 
\left( \bm e_i^{(t,k)} - \nabla f^{\mu}_i\left(\bm x_i^{(t,k)} \right)  \right) \right\|^2 \!\right] \nonumber \\
&=\frac{1}{MN} \sum_{i=1}^N \mathbb{E}_{\zeta_t} \left \|\sum_{k=0}^{H-1} 
\left( \bm e_i^{(t,k)} - \nabla f^{\mu}_i\left(\bm x_i^{(t,k)} \right)  \right) \right\|^2. 
\end{align}
\end{small}
According to \eqref{mini_batch_equal} and \cite[Lemma 2]{jianyu}, 
it follows that
\begin{small}
\equa{
T_4 =\frac{1}{MN} \sum_{i=1}^N \sum_{k=0}^{H-1}  \mathbb{E}_{\zeta_t^{k}} \left \|
 \bm e_i^{(t,k)} - \nabla f^{\mu}_i\left(\bm x_i^{(t,k)} \right) \right\|^2.
} 	
\end{small}
As $\mathbb{E} \|\bm z -\mathbb{E}[\bm z] \|^2  \leq \mathbb{E} \|\bm z\|^2 $, we have 
\begin{small}
\begin{align}\label{T_4_elem}
T_4  \leq  \frac{1}{MN}  \sum_{i =1}^N \sum_{k=0}^{H-1}  \mathbb{E}_{\zeta_t^{k}} \left \|
 \bm e_{i}^{(t,k)} \right\|^2 .
 \end{align}
\end{small}

Recalling \eqref{stochastic_estimator_ori} and \eqref{stochastic_estimator}, 
we have $\bm e_{i}^{(t,k)} =  \frac{1}{b_1 b_2} \sum_{m =1}^{b_1} \sum_{n=1}^{b_2}  \bm e_{i,m,n}^{(t,k)}$ by denoting 
\begin{small}
\begin{align}
\bm e_{i,m,n}^{(t,k)} =  \frac{d \bm v_{i,n}^{(t,k)} }{\mu}
	\left(\!
	 F_i (\bm x_i^{(t,k)} \!+\! \mu \bm v_{i,n}^{(t,k)} ,   \xi_{i,m}^{(t,k)} ) \!-\! F_i (\bm x_i^{(t,k)},  \xi_{i,m}^{(t,k)} )
	\!\right). \nonumber
  \end{align}
\end{small} 
According to \eqref{single_equal},
we have 
\equa{
\mathbb{E}_{t}^k \sbracket{\bm e_{i,m,n}^{(t,k)}} =  \nabla f_i^{\mu}\left(\bm x_i^{(t,k)} \right), ~\forall m,~n.
}
Therefore, we can bound $\mathbb{E}_{\zeta_t^{k}} \! \left  \|
\bm e_{i}^{(t,k)} \!\right\|^2$ as follows 
\begin{small}
\begin{align}\label{elem1}
 &\mathbb{E}_{\zeta_t^{k}} \! \left  \|
 \bm e_{i}^{(t,k)} \!\right\|^2 = \mathbb{E}_{\zeta_t^{k\!-\!1}} \left[ \mathbb{E}_{t}^k \left  \|
 \bm e_{i}^{(t,k)} \!\right\|^2 \right] \nonumber \\
 &= \! \mathbb{E}_{\zeta_t^{k\!-\!1}}\!\! \!\left[ \! \mathbb{E}_{t}^k\left \|\! \frac{1}{b_1 b_2} \!\sum_{m =1}^{b_1} \! \sum_{n=1}^{b_2} \bm e_{i,m,n}^{(t,k)} \!-\! \nabla f^{\mu}_i(\bm x_i^{(t,k)} ) \!\right \|^2 \!\!\!+\!  \left \|\! \nabla f^{\mu}_i(\bm x_i^{(t,k)} ) \!\right \|^2 \!\right] \nonumber \\
 &= \! \frac{1}{b_1 b_2} \mathbb{E}_{\zeta_t^{k}}\!  \left \|\! \bm e_{i,1,1}^{(t,k)} \!-\! \nabla f^{\mu}_i(\bm x_i^{(t,k)} ) \!\right \|^2 \!+\! \mathbb{E}_{\zeta_t^{k\!-\!1}}\left \|\! \nabla f^{\mu}_i(\bm x_i^{(t,k)} ) \!\right \|^2 \nonumber \\
 &\leq \! \frac{1}{b_1 b_2} \mathbb{E}_{\zeta_t^{k}}\!  \left \|\! \bm e_{i,1,1}^{(t,k)}  \!\right \|^2 \!+\!\mathbb{E}_{\zeta_t^{k-1}} \left \|\! \nabla f^{\mu}_i(\bm x_i^{(t,k)} ) \!\right \|^2,
\end{align}
\end{small}
\hspace{-0.25cm} where the second equality follows by 
\begin{small}
	 $\mathbb{E} \|\bm z\|^2 = \|\mathbb{E}[\bm z]\|^2  + \mathbb{E}\|\bm z - \mathbb{E}[\bm z]\|^2$,
 \end{small}
the third equality follows by the fact that $\bm e_{i,m,n}^{(t,k)}$ and $\bm e_{i,m^{\prime},n^{\prime}}^{(t,k)}$ are independent when $m \neq m^{\prime}$ or $n \neq n^{\prime}$, and the inequality holds because of $\mathbb{E} \|\bm z -\mathbb{E}[\bm z] \|^2  \leq \mathbb{E} \|\bm z\|^2 $.
We further bound $\mathbb{E}_{\zeta_t^{k}}\!  \left \|\! \bm e_{i,1,1}^{(t,k)}  \!\right \|^2$ as follows
\begin{small}
	\begin{align}\label{single_bound}
	\mathbb{E}_{\zeta_t^{k}}\!  \left \|\! \bm e_{i,1,1}^{(t,k)}  \!\right \|^2 \!& \leq 
	 \! 2d  \mathbb{E}_{\zeta_t^{k}} \! \left \|\nabla F_i(\bm x_i^{(t,k)}, \xi_{i,1}^{(t,k)}) \right \|^2 \!+\! \frac{1}{2} d^2 L^2\mu^2 \nonumber \\
	& \leq 
	2 c_g d \mathbb{E}_{\zeta_t^{k\!-\!1}}\! \left \|\!  \nabla f_i(\bm x_i^{(t,k)} )  \!\right \|^2 \!+\! 2d \sigma_g^2 \!+\! \frac{1}{2} d^2 L^2\mu^2,
	\end{align}
\end{small}
\hspace{-0.15cm}where the first inequality follows by \cite[Lemma 4.1]{GaoJZ18} and
the second inequality follows by Assumption \ref{bound_variance}.
Besides, 
\begin{small}
	\begin{align}\label{elem2}
	&~~~~\mathbb{E}_{\zeta_t^{k\!-\!1}} \left \| \nabla f^{\mu}_i(\bm x_i^{(t,k)} ) \right \|^2 \nonumber \\
	&= \mathbb{E}_{\zeta_t^{k\!-\!1}}\! \left \| \nabla f^{\mu}_i(\bm x_i^{(t,k)} ) \!-\!  \nabla f_i(\bm x_i^{(t,k)} )  \!+\!  \nabla f_i(\bm x_i^{(t,k)} )  \right \|^2 \nonumber \\
	& \leq  2\mathbb{E}_{\zeta_t^{k\!-\!1}}\left \| \nabla f^{\mu}_i(\bm x_i^{(t,k)} ) \!-\!  \nabla f_i(\bm x_i^{(t,k)} ) \right \|^2  \! + \! 2 \mathbb{E}_{\zeta_t^{k\!-\!1}} \!\left \|\!  \nabla f_i(\bm x_i^{(t,k)} )  \right \|^2 \nonumber \\
	& \leq 2\mu^2 L^2 + \! 2\mathbb{E}_{\zeta_t^{k\!-\!1}}\! \left \|  \nabla f_i(\bm x_i^{(t,k)} )  \right \|^2,
	\end{align}
\end{small}
\hspace{-0.2cm}where the last inequality follows by \cite[Lemma 2]{yi2021zeroth}.
By combining \eqref{elem1}, \eqref{single_bound}, and \eqref{elem2}, we have 
\begin{small}
	\begin{align}\label{mini_batch_bound}
	 \mathbb{E}_{\zeta_t^{k}} \! \left  \|
	\bm e_{i}^{(t,k)} \!\right\|^2  
 \leq  & \left(\! 2\!+\! \frac{2 c_g d}{b_1 b_2} \!  \right) \!  \mathbb{E}_{\zeta_t^{k\!-\!1}}  \left \|\nabla f_i(\bm x_i^{(t,k)}) \right \|^2 \nonumber \\
 &+\! \frac{2d\sigma_g^2}{b_1b_2}
	\!+\! \frac{d^2 L^2\mu^2}{2b_1b_2}\!+\! 2L^2\mu^2.
	\end{align}
\end{small}
     We next bound $\mathbb{E}_{\zeta_t^{k\!-\!1}}\! \left \|\!  \nabla f_i(\bm x_i^{(t,k)} )  \!\right \|^2$ as below
\begin{small}
\begin{align}\label{T_6_end}
&\mathbb{E}_{\zeta_t^{k\!-\!1}}\! \left \|\!  \nabla f_i(\bm x_i^{(t,k)} )  \!\right \|^2 \nonumber \\
=&  \mathbb{E}_{\zeta_t^{k\!-\!1}} \left \|\nabla f_i(\bm x_i^{(t,k)}) \mp \nabla f_i(\bm x^t) \mp \nabla f(\bm x^t) \right \|^2  \nonumber  \\
\leq & 3L^2 \mathbb{E}_{\zeta_t^{k\!-\!1}} \left\|\bm x_i^{(t,k)} - \bm x^t\right\|^2 + 3 \sigma_h^2 + (3c_h+3) \left \|\nabla f(\bm x^t) \right \|^2,
\end{align}
\vspace{-4mm}
\end{small}

\noindent where the inequality follows by the Cauchy-Schwartz inequality, Assumption \ref{assump_smooth}, and Assumption \ref{bound_heterogeneity}.
According to \eqref{T_4_elem}, \eqref{mini_batch_bound}, and \eqref{T_6_end}, we obtain 
\begin{small}
\begin{align}\label{T_4_end}
T_4 & \leq  \frac{6 \tilde{c}_g L^2 }{M} \mathbb{E}_{\zeta_t} \left[\! \frac{1}{N} \sum_{i=1}^N \sum_{k=0}^{H-1} \left\|\bm x_i^{(t,k)} - \bm x^t\right\|^2 \! \right]
\!+\! \frac{6\tilde{c}_g \tilde{c}_h H}{M} \left \|\nabla f(\bm x^t) \right \|^2 \nonumber  \\
& + \frac{2H}{M} \tilde{\sigma}^2 \!+\! \frac{2HL^2\mu^2}{M} \!+\!\frac{d^2 H L^2\mu^2}{2Mb_1b_2},
\end{align}
\end{small}
\hspace{-0.20cm}where $\tilde{\sigma}^2$, $\tilde{c}_g$, and $\tilde{c}_h$ are defined in Theorem \ref{theorem_full}.

Next, we split $T_5$ as follows
\begin{small}
\begin{align}\label{T_5_mid}
&T_5 = \mathbb{E}_{\zeta_t} \! \! \Bigg[ \!
\mathbb{E}_{\mathcal{M}_{t}} \left \| \frac{1}{M} \sum_{i \in \mathcal{M}_{t}} \sum_{k=0}^{H-1} 
\nabla f^{\mu}_i\left(\bm x_i^{(t,k)} \right) \right\|^2 \Bigg]= \nonumber \\
&
\underbrace{\mathbb{E}_{\zeta_t} \! \! \Bigg[ \!
\mathbb{E}_{\mathcal{M}_{t}}\! \left \| \frac{1}{M} \! \sum_{i \in \mathcal{M}_{t}} \!\! \sum_{k=0}^{H-1}  \!
	\nabla f^{\mu}_i\left(\bm x_i^{(t,k)} \right) \!-\! \frac{1}{N} \! \sum_{i=1}^N \!\! \sum_{k=0}^{H-1} \! \nabla f^{\mu}_i\left(\bm x_i^{(t,k)} \right) \right \|^2}_{T_6} 
 \nonumber \\
&
+ \left\|\frac{1}{N}\sum_{i=1}^N \sum_{k=0}^{H-1}  \nabla f_i^{\mu}\left(\bm x_i^{(t,k)} \right)\right\|^2  
\Bigg],       
\end{align}
\end{small}
\hspace{-0.23cm}
where the equality follows by 
$\mathbb{E} \|\bm z\|^2 \!=\! \|\mathbb{E} [\bm z]\|^2  \!+\! \mathbb{E}\|\bm z \!-\! \mathbb{E} [\bm z]\|^2$.

For $T_6$, we provide the following upper bounds   
\begin{small}
\begin{align}\label{T_7_end}
		T_6 \!= &\mathbb{E}_{t}\left  \| \frac{1}{M} \sum_{i \in \mathcal{M}_{t}} \sum_{k=0}^{H-1} \nabla f_i^{\mu}(\bm x_i^{(t,k)}) 
\! \mp \! \frac{1}{M} \sum_{i \in \mathcal{M}_{t}} \sum_{k=0}^{H-1} \nabla f_i(\bm x_i^{(t,k)})  \right . \nonumber \\
& \left. \! \mp \! \frac{1}{N}\sum_{i=1}^N \sum_{k=0}^{H-1} \nabla f_i(\bm x_i^{(t,k)} ) 
\!-\!\frac{1}{N}\sum_{i=1}^N \sum_{k=0}^{H-1} \nabla f^{\mu}_i (\bm x_i^{(t,k)} ) \right \|^2 \nonumber\\
\leq&3  \mathbb{E}_{t}\left  \|  \frac{1}{M} \sum_{i \in \mathcal{M}_{t}} \sum_{k=0}^{H-1} \left ( \nabla f_i^{\mu}(\bm x_i^{(t,k)}) \!-\!  \nabla f_i(\bm x_i^{(t,k)}) \right) \right \|^2 \nonumber\\
&+ 3  \mathbb{E}_{t} \left \| \frac{1}{M} \sum_{i \in \mathcal{M}_{t}} \sum_{k=0}^{H-1} \nabla f_i(\bm x_i^{(t,k)}) \!-\! \frac{1}{N}\sum_{i=1}^N \sum_{k=0}^{H-1} \nabla f_i (\bm x_i^{(t,k)} ) \right \|^2 \nonumber\\ 
&+ 3 \mathbb{E}_{t} \left \|\frac{1}{N}\sum_{i=1}^N \sum_{k=0}^{H-1} \left( \nabla f_i(\bm x_i^{(t,k)} ) \!-\! \nabla f^{\mu}_i(\bm x_i^{(t,k)} ) \right)\right \|^2 \nonumber\\
\leq&\underbrace{ 3  \mathbb{E}_{t} \left \| \frac{1}{M} \sum_{i \in \mathcal{M}_{t}} \sum_{k=0}^{H-1} \nabla f_i(\bm x_i^{(t,k)}) 
	\!-\! \frac{1}{N}\sum_{i=1}^N \sum_{k=0}^{H-1} \nabla f_i(\bm x_i^{(t,k)} ) \right \|^2}_{T_7} \nonumber  \\
& + 6H^2L^2\mu^2,
\end{align}
\end{small}
\hspace{-0.2cm}where the first inequality follows from the Cauchy-Schwartz inequality and the second inequality follows by  the Jensen's inequality and \cite[Lemma 2]{yi2021zeroth}. 

By substituting $\mp \frac{1}{M} \sum_{i \in \mathcal{M}_{t}} \sum_{k=0}^{H-1}  \nabla f_i\left(\bm x^t \right)$ and $\mp \frac{1}{N}\sum_{i=1}^N \sum_{k=0}^{H-1} \nabla f_i\left(\bm x^t \right)$
into $T_7$, and then following a similar derivation for bounding $T_6$, we can bound $T_7$ as follows
\begin{small}
\begin{align}\label{T_8_end}
 T_7  \leq & 18HL^2 \mathbb{E}_{\zeta_t} \left[\!\frac{1}{N}\sum_{i=1}^N \sum_{k=0}^{H-1} \left\|\bm x_i^{(t,k)} - \bm x^t\right\|^2 \!\right]\nonumber \\
&+ 9H^2 \underbrace{\mathbb{E}_{\mathcal{M}_t}\left \| \frac{1}{M} \sum_{i \in \mathcal{M}_{t}}  \nabla f_i\left(\bm x^t \right)  
\!-\! \nabla f\left(\bm x^t \right) \right \|^2}_{T_8}.
\end{align}
\end{small}

We continue by bounding $T_8$ as below 
\begin{small}
	\begin{align}\label{T_9_end}
T_8 &=  \mathbb{E}_{\mathcal{M}_t} \left \|\! \frac{1}{M} \sum_{i \in \mathcal{M}_{t}} \left( \nabla f_i\left(\bm x^t \right)  
\!-\! \nabla f\left(\bm x^t \right) \right) \right \|^2 \nonumber\\
&=  \frac{1}{M^2}~\mathbb{E}_{\mathcal{M}_t} \sum_{i \in \mathcal{M}_{t}} \left \| \nabla f_i\left(\bm x^t \right)  \!-\! \nabla f\left(\bm x^t \right) \right \|^2   \nonumber\\
&\leq  \frac{\sigma_h^2}{M} + \frac{c_h}{M} \left \|\nabla f\left(\bm x^t \right) \right \|^2,
\end{align}
\end{small}
\hspace{-0.15cm}where the second identity follows by the fact that devices in $\mathcal{M}_t$ are independently sampled from device set $\{1,2,\ldots,N\}$ and $\mathbb{E}_{i \sim \{1,2,\ldots,N\}} \sbracket{\nabla f_i\left(\bm x^t \right) }  = \nabla f\left(\bm x^t \right)$,
and the last inequality comes from Assumption \ref{bound_heterogeneity}. 

By combining  \eqref{T_5_mid}, \eqref{T_7_end}, \eqref{T_8_end}, and \eqref{T_9_end},
we bound $T_5$ as
\begin{small}
\begin{align}\label{T_5_end}
	&	T_5 \leq  18 H L^2 \mathbb{E}_{\zeta_t} \! \left[\! \frac{1}{N}\! \sum_{i=1}^N \! \sum_{k=0}^{H-1} \! \left\|\bm x_i^{(t,k)} - \bm x^t\right\|^2 \!\right] \!+\! \frac{9c_hH^2}{M} \left \|\! \nabla f\left(\bm x^t \right) \! \right \|^2   \nonumber \\
		& \!+\!  \frac{9H^2 \sigma_h^2}{M} \!+ 6H^2L^2\mu^2 \!+\! \mathbb{E}_{\zeta_t}  \left\|\frac{1}{N}\sum_{i=1}^N \sum_{k=0}^{H-1} \nabla f_i^{\mu}\left(\bm x_i^{(t,k)} \right)\right\|^2.
\end{align}
\end{small}
 By combining \eqref{descent_lemma_1}, \eqref{T_1_end}, \eqref{T_2_mid}, \eqref{T_4_end}, and \eqref{T_5_end}, 
 we obtain Lemma \ref{progress_lem_part}.

 \subsubsection{Proof of Lemma \ref{progress_lem_full}}\label{appen_proof_lem_1}
Most of the intermediary results for proving Lemma \ref{progress_lem_part} in Appendix \ref{appen_proof_lem_3} can be directly applied here by replacing $M$ with $N$. The main difference is that there is no randomness in $\mathcal{M}_{t}$, and we thus do not need to construct an upper bound for $T_5$ in \eqref{T_2_mid}.
By combining \eqref{descent_lemma_1}, \eqref{T_1_end}, \eqref{T_2_mid}, and \eqref{T_4_end},
we obtain Lemma \ref{progress_lem_full}.

\subsubsection{Proof of Lemma \ref{drift_lem}}\label{appen_proof_lem_2}
By denoting $\frac{1}{N}\sum_{i=1}^N \mathbb{E}_{\zeta_t^{k\!-\!1}} \left\|\bm x_i^{(t,k)}  - \bm x^t\right\|^2$ as $s^{(t,k)}$, we have	
\begin{small}
\begin{align}\label{drift_tau}
s^{(t,\tau)} 
&= \eta^2 \frac{1}{N}\sum_{i=1}^N \mathbb{E}_{\zeta_t^{\tau\!-\!1}} \left\|  \sum_{k=0}^{\tau-1} \bm e_i^{(t,k)} \right\|^2   \nonumber \\
&\leq  \tau \eta^2 \sum_{k=0}^{\tau-1} \frac{1}{N}\sum_{i=1}^N 
\mathbb{E}_{\zeta_t^k} \left\|  \bm e_i^{(t,k)} \right\|^2,
\end{align}	
\end{small}
\hspace{-0.23cm} where the inequality follows by the Cauchy-Schwartz inequality. 
By combining \eqref{mini_batch_bound}, \eqref{T_6_end}, and \eqref{drift_tau}, we have
	\begin{align}\label{H_step_derive}
s^{(t,\tau)} 
\leq 
& 6\tilde{c}_gL^2 \tau \eta^2  \sum_{k=0}^{\tau-1} s^{(t,k)} + 6 \tau^2 \eta^2\tilde{c}_g \tilde{c}_h \left \|\nabla f(\bm x^t) \right \|^2  \nonumber \\
& +\! 2\tau^2 \eta^2 \tilde{\sigma}^2 \!+\! 2 L^2\tau^2 \eta^2\mu^2 \! +\! \frac{d^2 L^2 \tau^2}{2b_1b_2}\eta^2\mu^2.
\end{align}	
By taking summation over $\tau$ from $1$ to $H\!-\!1$, we obtain 
\begin{align}\label{iterate_s}
\sum_{\tau=1}^{H-1} s^{(t,\tau)} 
&\leq  6\tilde{c}_g L^2 \eta^2  \sum_{\tau=1}^{H-1} \tau  \sum_{k=0}^{\tau-1} s^{(t,k)} 
+ C_0 \nonumber \\
&\leq 3 \tilde{c}_g H^2 L^2 \eta^2  \sum_{k=0}^{H-1} s^{(t,k)} 
+ C_0,
\end{align}	
where we utilize the property of arithmetic sequence and
\equa{
C_0 =& 2 H^3 \eta^2\tilde{c}_g \tilde{c}_h \left \|\nabla f(\bm x^t) \right \|^2 \!+\! \frac{2}{3}H^3 \eta^2 \tilde{\sigma}^2 \!+\! \frac{2}{3} L^2 H^3 \eta^2\mu^2 \nonumber \\
& + \frac{d^2 L^2 H^3}{6b_1b_2}\eta^2\mu^2.
}
As $s^{(t,0)}\!=\!0$ and by rearranging \eqref{iterate_s}, we have
\begin{align}
&\bracket{1-3 \tilde{c}_g H^2 L^2 \eta^2} \sum_{\tau=0}^{H-1} s^{(t,\tau)} 
\leq  2 H^3 \eta^2\tilde{c}_g \tilde{c}_h \left \|\nabla f(\bm x^t) \right \|^2 \nonumber \\
& +\! \frac{2}{3} L^2 H^3 \eta^2\mu^2 \!+\! \frac{2}{3}H^3 \eta^2 \tilde{\sigma}^2 + \frac{d^2 L^2 H^3}{6b_1b_2}\eta^2\mu^2.
\end{align}	
As $\eta \leq \frac{1}{3HL\sqrt{\tilde{c}_g}}$, we have $3(1-3 \tilde{c}_g H^2 L^2 \eta^2) \geq 2$. Note that $\sum_{\tau=0}^{H-1} s^{(t,\tau)} = \delta_t$.
We thus obtain Lemma \ref{drift_lem}.

\subsubsection{Proof of Lemma \ref{drift_lem_H}}\label{appen_proof_lem_4}
Lemma \ref{drift_lem_H} is a byproduct of the derivation of Lemma \ref{drift_lem}. It follows from \eqref{H_step_derive} by setting $\tau=H$.

\ifCLASSOPTIONcaptionsoff
  \newpage
\fi

\bibliographystyle{IEEEtran}  
\bibliography{reference.bib}

\end{document}